\documentclass[10pt, a4paper]{amsart}

\input{style}

\begin{document}

\title{A Statistical Test for Probabilistic Fairness}
\date{}
\author{Bahar Taskesen, Jos\'{e} Blanchet, Daniel Kuhn, Viet Anh Nguyen}
\thanks{The authors are with the Risk Analytics and Optimization Chair, EPFL, Switzerland (\texttt{bahar.taskesen, daniel.kuhn@epfl.ch}) and the Department of Management Science and Engineering, Stanford University (\texttt{jose.blanchet, viet-anh.nguyen@stanford.edu}).}
\begin{abstract}
Algorithms are now routinely used to make consequential decisions that affect human lives. Examples include college admissions, medical interventions or law enforcement. While algorithms empower us to harness all information hidden in vast amounts of data, they may inadvertently amplify existing biases in the available datasets. This concern has sparked increasing interest in fair machine learning, which aims to quantify and mitigate algorithmic discrimination. Indeed, machine learning models should undergo intensive tests to detect algorithmic biases before being deployed at scale. In this paper, we use ideas from the theory of optimal transport to propose a statistical hypothesis test for detecting unfair classifiers. Leveraging the geometry of the feature space, the test statistic quantifies the distance of the empirical distribution supported on the test samples to the manifold of distributions that render a pre-trained classifier fair. We develop a rigorous hypothesis testing mechanism for assessing the probabilistic fairness of any pre-trained logistic classifier, and we show both theoretically as well as empirically that the proposed test is asymptotically correct. In addition, the proposed framework offers interpretability by identifying the most favorable perturbation of the data so that the given classifier becomes fair. 
\end{abstract}
	\maketitle

\section{Introduction}
\label{sec:intro}



The past decade witnessed data and algorithms becoming an integrative part of the human society. Recent technological advances are now allowing us to collect and store an astronomical amount of unstructured data, and the unprecedented computing power is enabling us to convert these data into decisional insights. Nowadays, machine learning algorithms can uncover complex patterns in the data to produce an exceptional performance that can match, or even surpass, that of humans. These algorithms, as a consequence, are proliferating in every corner of our lives, from suggesting us the next vacation destination to helping us create digital paintings and melodies. Machine learning algorithms are also gradually assisting humans in consequential decisions such as deciding whether a student is admitted to college, picking which medical treatment to be prescribed to a patient, and determining whether a person is convicted. Arguably, these decisions impact radically many people's lives, together with the future of their loved ones.

Algorithms are conceived and function following strict rules of logic and algebra; it is hence natural to expect that machine learning algorithms deliver objective predictions and recommendations. Unfortunately, in-depth investigations reveal the excruciating reality that state-of-the-art algorithmic assistance is far from being free of biases. For example, a predictive algorithm widely used in the United States criminal justice system is more likely to \textit{mis}classify African-American offenders into the group of high recidivism risk compared to white-Americans~\cite{chouldechova2017fair, ref:propublica}. The artificial intelligence tool developed by Amazon also learned to penalize gender-related keywords such as ``women's'' in the profile screening process, and thus may prefer to recommend hiring male candidates for software development and technical positions~\cite{ref:dastin2018amazon}. Further, Google’s ad-targeting algorithm displayed advertisements for higher-paying executive jobs more often to men than to women  \cite{ref:datta2015automated}.

There are several possible explanations for why cold, soulless algorithms may trigger biased recommendations. First, the data used to train machine learning algorithms may already encrypt human biases manifested in the data collection process. These biases arise as the result of a suboptimal design of experiments, or from historically biased human decisions that accumulate over centuries. Machine-learned algorithms, which are apt to detect underlying patterns from data, will unintentionally learn and maintain these existing biases~\cite{ref:buolamwini2018gender,  ref:manrai2016genetic}. For example, secretary or primary school teacher are  professions which are predominantly taken by women, thus, natural language processing systems are inclined to associate female attributes to these jobs. Second, training a machine learning algorithm typically involves minimizing the prediction error which privileges the majority populations over the minority groups. Clinical trials, for instance, typically involve very few participants from the minority groups such as indigenous people, and thus medical interventions recommended by the algorithms may not align perfectly to the characteristics and interests of patients from the minority groups. Finally, even when the sensitive attributes are not used in the training phase, strong correlations between the sensitive attributes and the remaining variables in the dataset may be exploited to generate unjust actions. For example, the sensitive attribute of race can be easily inferred with high accuracy based on common non-sensitive attributes such as the travel history of passengers or the grocery shopping records of customers.


The pressing needs to redress undesirable algorithmic biases have propelled the rising field of fair machine learning\footnote{Comprehensive surveys on fair machine learning can be found in \cite{ref:berk2018fairness, ref:chouldechova2020snapshot, ref:corbett2017algorithmic, ref:mehrabi2019survey}.}. A building pillar of this field involves the verification task: given a machine learning algorithm, we are interested in verifying if this algorithm satisfies a chosen criterion of fairness. This task is performed in two steps: first, we choose an appropriate notion of fairness, then the second step invokes a computational procedure, which may or may not involve data, to decide if the chosen fairness criterion is fulfilled. A plethora of criteria for fair machine learning were proposed in the literature, many of them are motivated by philosophical or sociological ideologies or legal constraints. For example, anti-discrimination laws may prohibit making decisions based on sensitive attributes such as age, gender, race or sexual orientation. Thus, a na\"{i}ve strategy, called fairness through unawareness, involves removing all sensitive attributes from the training data. However, this strategy seldom guarantees any fairness due to the inter-correlation issues~\cite{ref:grgic2016case, ref:garg2019counterfactual}, and thus potentially fails to generate inclusive outcomes~\cite{ref:barocas2016big, ref:black2020fliptest, ref:kleinberg2018algorithmic, ref:lipton2018does}. Other notions of fairness aim to either promote individual fairness~ \cite{ref:dwork2012fairness}, prevent disparate treatment \cite{ref:zafar2017fairness} or avoid disparate mistreatment~\cite{ref:feldman2015certifying, ref:zafar2015fairness} of the algorithms. Towards similar goals, notions of \textit{group} fairness focus on reducing the difference of favorable outcomes proportions among different sensitive groups. Examples of group fairness notions include
disparate impact~\cite{ref:zafar2017fairness},~demographic parity (statistical parity)~\cite{ref:calders2010three, ref:dwork2012fairness},~equality of opportunity~\cite{ref:hardt2016equality} and equalized odds~\cite{ref:hardt2016equality}. The notion of counterfactual fairness \cite{ref:garg2019counterfactual} was also suggested as a measure of causal fairness. Despite the abundance of available notions, there is unfortunately no general consensus on the most suitable measure to serve as the industry standard. Moreover, except in trivial cases, it is not possible for a machine learning algorithm to simultaneously satisfy multiple notions of fairness \cite{ref:berk2018fairness, kleinberg2016inherent}. Therefore, the choice of the fairness notion is likely to remain more an art than a science.

This paper focuses not on the normative approach to choosing an ideal notion of machine learning fairness. We endeavor in this paper to shed more light on the computational procedure to complement the verification task. Concretely, we position ourselves in the classification setting, which is arguably the most popular task in machine learning. Moreover, we will focus on notions of group fairness, and we employ the framework of statistical hypothesis test instead of algorithmic test.

    \noindent \textbf{Contributions.} Our paper makes two concrete contributions to the problem of fairness testing of machine learning's classifiers.
    \begin{enumerate}[leftmargin = 5mm]
        \item We propose the Wasserstein projection framework to perform statistical hypothesis test of group fairness for classification algorithms. 
        We derive in details the computation of the test statistic and the limiting distribution when fairness is measured using the probabilistic equality of opportunity and probabilistic equalized odds criteria.
        \item We demonstrate that the Wasserstein projection hypothesis testing paradigm is asymptotically correct and can exploit additional information on the geometry of the feature space. Moreover, we also show that this paradigm promotes transparency and interpretability through the analysis of the most favorable distributions.
    \end{enumerate}
    
    The remaining of the paper is structured as follows. In Section~\ref{sec:framework}, we introduce the general problem of statistical hypothesis test of classification fairness, and depict the current landscape of fairness testing in the literature. Section~\ref{sec:approach} details our Wasserstein projection approach to this problem. Sections~\ref{sec:EO} and~\ref{sec:EOdd} apply the proposed framework to test if a pre-trained logistic classifier satisfies the fairness notion of probabilistic equal opportunity and probabilistic equalized odds, respectively. Numerical experiments are presented in Section~\ref{sec:experiment} to empirically validate the correctness and demonstrate the power of our proposed paradigm. Section~\ref{sec:conclusion} concludes the paper with outlooks on the broader impact of our Wasserstein projection hypothesis testing approach. 
    
    All technical proofs are relegated to the Appendix.
\section{Statistical Testing Framework for Fairness and Literature Review}
\label{sec:framework}

We consider throughout this paper a generic binary classification setting. Let $\mc X = \R^d$ and $\mc Y=\{0, 1\}$ be the space of feature inputs and label outputs of interest. We assume that there is a single sensitive attribute corresponding to each data point and its space is denoted by $\mc A= \{0, 1\}$. A probabilistic classifier is represented by a function $h(\cdot) : \mc X \to [0, 1]$ that outputs for each given sample $x \in \mc X$ the probability that $x$ belongs to the positive class. The deterministic classifier predicts class 1 if $h(x) \ge \tau$ and class 0 otherwise, where $\tau \in [0, 1]$ is a classification threshold. Note that the function $h$ depends only on the feature $X$, but not on the sensitive attribute $A$, thus predicting $Y$ using $h$ satisfies fairness through unawareness. 

The central goal of this paper is to provide a statistical test to detect if a classifier $h$ fails to satisfy a prescribed notion of machine learning fairness. A statistical hypothesis test can be cast with the null hypothesis being
\begin{center}
    $\mc H_0$: the classifier $h$ is fair,
\end{center}
against the alternative hypothesis being
\begin{center}
    $\mc H_1$: the classifier $h$ is not fair.
\end{center}
In this paper, we focus on statistical notions of \textit{group} fairness, which are usually defined using conditional probabilities. A prevalent notion of fairness in machine learning is the criterion of equality of opportunity\footnote{We use two terms ``equality of opportunity'' and ``equal opportunity'' interchangeably.}, which requires that the true positive rate are equal between subgroups.
\begin{definition}[Equal opportunity \cite{ref:hardt2016equality}] \label{def:EO}
    A classifier $h(\cdot):\mc X \to [0, 1]$ satisfies the equal opportunity criterion relative to $\QQ$ if
    \[
        \QQ(h(X) \ge \tau | A = 1, Y = 1) = \QQ(h(X) \ge \tau | A = 0, Y = 1),
    \]
    where $\tau$ is the classification threshold.
\end{definition}
Another popular criterion of machine learning fairness is the equalized odds, which is more stringent than the equality of opportunity: it requires that the positive outcome is conditionally independent of the sensitive attributes given the true label.
\begin{definition}[Equalized odds \cite{ref:hardt2016equality}] 
\label{def:EOdd}
    A classifier $h(\cdot):\mc X \to [0, 1]$ satisfies the equalized odds criterion relative to $\QQ$ if
    \[
        \QQ(h(X) \ge \tau | A = 1, Y = y) = \QQ(h(X) \ge \tau | A = 0, Y = y)  ~ \forall y \in \mc Y,
    \]
    where $\tau$ is the classification threshold.
\end{definition}
Notice that the criteria of fairness presented in Definitions~\ref{def:EO} and~\ref{def:EOdd} are dependent on the distribution $\QQ$: a classifier $h$ can be fair relative to a distribution $\QQ_1$, but it may become unfair with respect to another distribution $\QQ_2 \neq \QQ_1$. If we denote by $\PP$ the true population distribution that governs the random vector $(X, A, Y)$, then it is imperative and reasonable to test for group fairness with respect to $\PP$. For example, to test for the equality of opportunity, we can reformulate a two-sample equal conditional mean test of the null hypothesis
    \[
    \mc H_0: \EE_\PP[\mathbbm{1}_{h(X)\ge\tau} | A = 1, Y= 1] = \EE_\PP[\mathbbm{1}_{h(X)\ge\tau}| A= 0, Y = 1],
    \]
    and one can potentially employ a Welch's $t$-test with proper adjustment for the randomness of the sample size. Unfortunately, deriving the test becomes complicated when the null hypothesis involves an equality of multi-dimensional quantities, which arises in the case of equalized odds, due to the complication of the covariance terms. Variations of the permutation tests were also proposed to detect discriminatory behaviour of machine learning algorithms following the same formulation of the one-dimensional two-sample equality of conditional mean test \cite{ref:diciccio2020evaluating, ref:tramer2017fairtest}. However, these permutation tests follow a black-box mechanism and are unable to be generalized to multi-dimensional tests. Tests based on group fairness notions can also be accomplished using an algorithmic approach as in~\cite{ref:diciccio2020evaluating, ref:saleiro2018aequitas, ref:del2018obtaining}.

    From a broader perspective, deriving tests for fairness is an active area of research, and many testing procedures have been recently proposed to test for individual fairness \cite{ref:xue2020auditing, ref:john2020verifying}, for counterfactual fairness \cite{ref:black2020fliptest, ref:garg2019counterfactual} and diverse other criteria \cite{ref:bellamy2018ai, ref:wexler2019if, ref:tramer2017fairtest}.

\noindent \textbf{Literature related to optimal transport.} Optimal transport is a long-standing field that dates back to the seminal work of Gaspard Monge~\cite{monge1781memoire}. In the past few years, it has attracted significant attention in the machine learning and computer science communities thanks to the availability of fast approximation algorithms~\cite{sinkhorn, dvurechensky2018computational, benamou2015iterative, blondel2017smooth, genevay2016stochastic}. Optimal transport is particularly successful in various learning tasks, notably generative mixture models \cite{kolouri2017optimal, nguyen2013convergence}, image processing \cite{alvarez2017structured, ferradans2014regularized, kolouri2015transport, papadakis2017convex, tartavel2016wasserstein}, computer vision and graphics \cite{pele2008linear, pele2009fast, rubner2000earth, solomon2014earth, solomon2015convolutional}, clustering \cite{ho2017multilevel}, dimensionality reduction \cite{cazelles2018geodesic, flamary2018wasserstein, rolet2016fast, schmitzer2016sparse, seguy2015principal}, domain adaptation \cite{courty2016optimal, murez2018image}, signal processing \cite{thorpe2017transportation} and data-driven distributionally robust optimization~\cite{ref:kuhn2019wasserstein, ref:blanchet2019robust, ref:gao2017wasserstein, ref:zhao2018data}. Recent comprehensive survey on optimal transport and its applications can be found in~\cite{peyre2019computational, kolouri2017optimal}.

In the context of fair classification, ideas from optimal transport have been used to construct fair logistic classifier~\cite{ref:taskesen2020distributionally}, to detect classifiers that does not obey group fairness notions, or to ensure fairness by pre-processing \cite{ref:del2018obtaining}, to learn a fair subspace embedding that promotes fair classification~\cite{ref:yurochkin2020training}, to test individual fairness~\cite{ref:xue2020auditing}, or to construct a counterfactual test \cite{ref:black2020fliptest}.

\section{Wasserstein Projection Framework for Statistical Test of Fairness}
\label{sec:approach}

We hereby provide a fresh alternative to the testing problem of machine learning fairness. On that purpose, for a given classifier $h$, we define abstractly the following set of distributions
\be \label{eq:F-def}
    \mc F_h = \left\{ \QQ \in \mc P:~ \text{ the classifier $h$ is fair relative to } \QQ \right\},
\ee
where $\mc P$ denotes the space of all distributions on $\mc X \times \mc A \times \mc Y$. Intuitively, the set $\mc F_h$ contains all probability distributions under which the classifier $h$ satisfies the prescribed notion of fairness. It is trivial to see that if $\mc F_h$ contains the true data-generating distribution $\PP$, then the classifier $h$ is fair relative to $\PP$. Thus, we can reinterpret the hypothesis test of fairness using the hypotheses
\begin{center}
    $\mc H_0$: $\PP \in \mc F_h$, \hspace{1cm} $\mc H_1$: $\PP \not \in \mc F_h$.
\end{center}
Testing the inclusion of $\PP$ in $\mc F_h$ is convenient if $\mc P$ is endowed with a distance.
In this paper, we equip $\mc P$  with the Wasserstein distance.
\begin{definition}[Wasserstein distance]
    The type-$2$ Wasserstein distance between two probability distributions $\QQ$ and $\QQ'$ supported on $\Xi$ is defined as 
    \[
    \Wass(\QQ',\QQ) = \Min{\pi  \in \Pi(\QQ', \QQ)} \sqrt{\EE_\pi[ c(\xi',\xi)^2]},
    \] 
    where the set $\Pi (\QQ', \QQ)$ contains all joint distributions of the random vectors $\xi'\in \Xi$ and $\xi\in \Xi$ under which $\xi'$ and $\xi$ have marginal distributions $\QQ'$ and $\QQ$, respectively, and $c:\Xi\times\Xi\rightarrow [0,\infty]$ constitutes a lower semi-continuous ground metric.
\end{definition}
The type-2 Wasserstein distance\footnote{From this point, we omit the term ``type-2'' for brevity.} is a special instance of the optimal transport. The squared Wasserstein distance between $\QQ'$ and $\QQ$ can be interpreted as the cost of moving the distribution $\QQ'$ to $\QQ$, where $c(\xi', \xi)$ is the cost of moving a unit mass from $\xi'$ to $\xi$. Being a distance on $\mc P$, $\Wass$ is symmetric, non-negative and vanishes to zero if $\QQ' = \QQ$. The Wasserstein distance is hence an attractive measure to identify if $\PP$ belongs to $\mc F_h$. Using this insight, the hypothesis test for fairness has the equivalent representation 
\begin{center}
    $\mc H_0$: $\inf_{\QQ \in \mc F_h} \Wass(\PP, \QQ) = 0$, \hspace{0.5cm} $\mc H_1$: $\inf_{\QQ \in \mc F_h} \Wass(\PP, \QQ) > 0$.
\end{center}

Even though $\PP$ remains elusive to our knowledge, we are given access to a set of i.i.d~test samples $\{(\wh x_i, \wh a_i, \wh y_i)\}_{i=1}^N$ generated from the true distribution $\PP$. Thus we can rely on the empirical value
\[
\inf_{\QQ \in \mc F_h}~\Wass(\Pnom^N, \QQ),
\]
which is the distance from the empirical distribution supported on the samples $\Pnom^N = \sum_{i=1}^N \delta_{(\wh x_i, \wh a_i, \wh y_i)}$ to the set $\mc F_h$. To perform the test, it is sufficient to study the limiting distribution of the test statistic using proper scaling under the null hypothesis $\mc H_0$. The outcome of the test is determined by comparing the test statistic to the quantile value of the limiting distribution at a chosen level of significant $\alpha \in (0, 1)$.

\noindent \textbf{Advantages.} The Wasserstein projection framework to hypothesis testing that we described above offers several advantages over the existing methods.
\begin{enumerate}[leftmargin=5mm]
    \item Geometric flexibility: The definition of the Wasserstein distance implies that there exists a joint ground metric $c$ on the space of the features, the sensitive attribute and the label. If the modelers or the regulators possess any structural information on an appropriate metric on $\Xi = \mc X \times \mc A \times \mc Y$, then this information can be exploited in the testing procedure. Thus, the Wasserstein projection framework equips the users with an additional freedom to inject prior geometric  information into the statistical test.
    \item Mutivariate generalizability: Certain notions of fairness, such as equalized odds, are prescribed using multiple equalities of conditional expectations. The Wasserstein projection framework encapsulates these equalities simultaneously in the definition of the set $\mc F_h$, and provides a \textit{joint} test of these equalities without the hassle of decoupling and testing individual equalities as being done in the currently literature.
    \item Interpretability: If we denote by $\QQ\opt$ the projection of the empirical distribution $\Pnom^N$ onto the set of distributions $\mc F_h$, i.e.,
    \[
        \QQ\opt = \arg \Min{\QQ \in \mc F_h}~\Wass(\Pnom^N, \QQ),
    \]
    then $\QQ\opt$ encodes the minimal perturbation to the empirical samples so that the classifier $h$ becomes fair. The distribution $\QQ\opt$ is thus termed the most favorable distribution, and examining $\QQ\opt$ can reveal the underlying mechanism and explain the outcome of the hypothesis test. The accessibility to $\QQ\opt$ showcases the expressiveness of the Wasserstein projection framework. 
\end{enumerate}

Whilst theoretically sound and attractive, there are three potential difficulties with the Wasserstein projection approach to statistical test of fairness.
First, to project $\Pnom^N$ onto the set $\mc F_h$, we need to solve an infinite-dimensional optimization problem, which is inherently difficult. Second, for many notions of machine learning fairness such as the equality of opportunity and the equalized odds, the corresponding set $\mc F_h$ in~\eqref{eq:F-def} is usually prescribed using \textit{non}linear constraints. For example, if we consider the equal opportunity criterion in Definition~\ref{def:EO}, then the set $\mc F_h$ can be re-expressed using a fractional function of the probability measure as
    \begin{align*}
        \mc F_h =\left\{ \begin{array}{l}
            \QQ \in \mc P \text{ such that } \ds \frac{\QQ(h(X) \ge \tau, A = 1, Y = 1)}{\QQ(A = 1, Y=1)} = \frac{\QQ(h(X) \ge \tau, A = 0, Y = 1)}{\QQ(A = 0, Y=1)}
        \end{array}
        \right\}.
    \end{align*}
    Apart from involving nonlinear constraints, it is easy to verify that the set $\mc F_h$ is also non-convex, which amplifies the difficulty of computing the projection onto $\mc F_h$.
Finally, the limiting distribution of the test statistic is difficult to analyze due to the discontinuity of the probability function at the set $\{x \in \mc X: h(x) = \tau\}$. The asymptotic analysis with this discontinuity is of a combinatorial nature, and is significantly more problematic than the asymptotic analysis of smooth quantities.

While these difficulties may be overcome via various ways, in this paper we choose the following combination of remedies. First, we will use a relaxed notion of fairness termed \textit{probabilistic fairness}, which was originally introduced in \cite{ref:pleiss2017fairness}. 
Second, when computing the Wasserstein distances between distributions on $\mc X \times \mc A \times \mc Y$, we use
\be
\label{eq:cost}
    c\big( (x', a', y'),  (x, a, y) \big) = \| x - x'\| + \infty | a - a'| + \infty | y - y'|
\ee
as the ground metric, where $\|\cdot\|$ is a norm on $\R^d$. This case corresponds to having an absolute trust in the label and in the sensitive attribute of the training samples. This absolute trust restriction is common in the literature of fair machine learning~\cite{ref:xue2020auditing, ref:taskesen2020distributionally}. 

We now briefly discuss the advantage of using the ground metric of the form~\eqref{eq:cost}. Denote by $p \in \R_{++}^{|\mc A| \times |\mc Y|}$ the array of the true marginals of $(A, Y)$, in particular, $ p_{{a} {y}} = \PP(A = a, Y = y)$ for all $a \in \mc A$ and $y \in \mc Y$.
Further, let $\wh p^N \in \R_{++}^{|\mc A| \times |\mc Y|}$ be the array of the empirical marginals of $(A, Y)$ under the empirical measure $\Pnom^N$, that is, $\wh p^N_{{a} {y}} = \Pnom^N(A = a, Y = y)$ for all $a \in \mc A$ and $y \in \mc Y$.
Throughout this paper, we assume that the empirical marginals are proper, that is, $\wh p_{ay}^N \in (0, 1)$ for any $(a, y) \in \mc A \times \mc Y$.
We define temporarily the simplex set $\Delta \defeq \{ \bar p\in \R_{++}^{|\mc A|\times |\mc Y|}:\sum_{a \in \mc A, y \in \mc Y} \bar p_{ay} = 1 \}$.
Subsequently, for any marginals $\bar p \in \Delta$, we define the marginally-constrained set of distributions
\[
    \mc F_h(\bar p) \Let \left\{ \QQ \in \mc P:
    \begin{array}{l} h \text{ is fair relative to } \QQ \\
    \QQ(A = a, Y = y) = \bar p_{ay} ~~ \forall (a, y) \in \mc A \times \mc Y
    \end{array}
    \right\}.
\]
Using these notations, one can readily verify that
\[
    \mc F_h= \cup_{\bar p \in \Delta} \mc F_h(\bar p).
\]
Moreover, the next result asserts that in order to compute the projection of $\Pnom^N$ onto $\mc F_h$, to suffices to project onto the marginally-constrained set $\mc F_h(\wh p^N)$.
\begin{lemma}[Projection with marginal restrictions] \label{lemma:marginal}
    Suppose that the ground metric is chosen as in~\eqref{eq:cost}. If a measure $\QQ \in \mc F_h$ satisfies $\Wass(\Pnom^N, \QQ) < \infty$, then $\QQ \in \mc F_h(\wh p^N)$.   
\end{lemma}

A useful consequence of Lemma~\ref{lemma:marginal} is that 
\be \label{eq:relation}
\inf_{\QQ \in \mc F_h}~\Wass(\Pnom^N, \QQ) =  \inf_{\QQ \in \mc F_h(\wh p^N)}~\Wass(\Pnom^N, \QQ),
\ee
where the feasible set of the problem on the right-hand side is the marginally-constrained set $\mc F_h(\wh p^N)$ using the empirical marginals $\wh p^N$. For two notions of probabilistic fairness that we will explore in this paper, projecting $\Pnom^N$ onto $\mc F_h(\wh p^N)$ is arguably easier than onto $\mc F_h$. Thus, this choice of ground metric improves the  tractability when computing the test statistic.

Third, and finally, we will focus on the logistic regression setting, which is one of the most popular classification methods \cite{ref:hosmer2013applied}. In this setting, the conditional probability $\PP[Y=1 | X=x]$ is modelled by the sigmoid function
\[h_\beta(x) = \frac{1}{1 + \exp(-\beta^\top x)},\]
where $\beta \in \R^d$ is the regression parameter. Moreover, a classifier with $\beta = 0$, is trivially fair. Thus, it suffices to consider $\beta \neq 0$.

\noindent \textbf{Notations. } We use $\|\cdot\|_*$ to denote the dual norm of $\|\cdot\|$. For any integer $N$, we define $[N] \defeq \{1, 2, \ldots, N\}$. Given $N$ test samples $(\wh x_i, \wh a_i, \wh y_i)_{i=1}^N$, we use $\mc I_y \Let \{ i \in [N]: \wh y_i = y\}$ to denote the index set of observations with label $y$. The parameters $\lambda_i$ are defined as
    \be \label{eq:lambdai-def}
    \forall i \in [N]: \quad \lambda_i = \begin{cases}
        (\wh p_{11}^N)^{-1} & \text{if } (\wh a_i, \wh y_i) = (1, 1),\\
        - (\hat p_{01}^N)^{-1} &\text{if } (\wh a_i, \wh y_i) = (0, 1), \\
        (\wh p_{10}^N)^{-1} & \text{if } (\wh a_i, \wh y_i) = (1, 0),\\
        - (\hat p_{00}^N)^{-1} &\text{if } (\wh a_i, \wh y_i) = (0, 0).
    \end{cases}
    \ee
\section{Testing Fairness for Probabilistic Equal Opportunity Criterion}
\label{sec:EO}

In this section, we use the ingredients introduced in the previous section to concretely construct a statistical test for the fairness of a logistic classifier $h_\beta$. Specifically, we will employ the probabilistic equal opportunity criterion which was originally proposed in~\cite{ref:pleiss2017fairness}.
\begin{definition}[Probabilistic equal opportunity criterion \cite{ref:pleiss2017fairness}]
\label{def:proba-EO}
    A logistic classifier $h_\beta:\mc X \to [0, 1]$ satisfies the probabilistic equalized opportunity criteria relative to a distribution $\QQ$ if
    \[
        \EE_{\QQ}[h_\beta(X) | A = 1, Y = 1] = \EE_{\QQ}[h_\beta(X) | A = 0, Y = 1].
    \]
\end{definition}

The probabilistic equal opportunity criterion, which serves as a surrogate for the equal opportunity criterion in Definition~\ref{def:EO},  depends on the smooth and bounded sigmoid function $h_\beta$ but is independent of the classification threshold $\tau$. 
Motivated by \cite{ref:lohaus2020too}, we empirically illustrate in Figure~\ref{fig:unfairness_landscape} that the probabilistic surrogate provides a good approximation of the equal opportunity criterion. Figure~\ref{fig:unfairness_landscape-opp} plots the absolute difference of the classification probabilities $|\PP(h(X) \ge \half | A=1, Y=1)-\PP(h(X) \ge \half | A=0, Y=1)|$, while Figure~\ref{fig:unfairness_landscape-probopp} plots the absolute difference of the sigmoid expectations $|\EE_\PP[h(X)  | A=1, Y=1]-\EE_\PP[h(X) | A=0, Y=1]|$. One may observe that the regions of $\beta$ so that the absolute differences fall close to zero are similar in both plots. This implies that a logistic classifier $h_\beta$ which is equal opportunity fair is also likely to be \textit{probabilistic} equal opportunity fair, and vice versa.

\begin{figure}[h!]
    \centering
    \hspace{-4mm}
    \begin{subfigure}[t]{0.35 \columnwidth}
    \centering
    \includegraphics[width=\columnwidth]{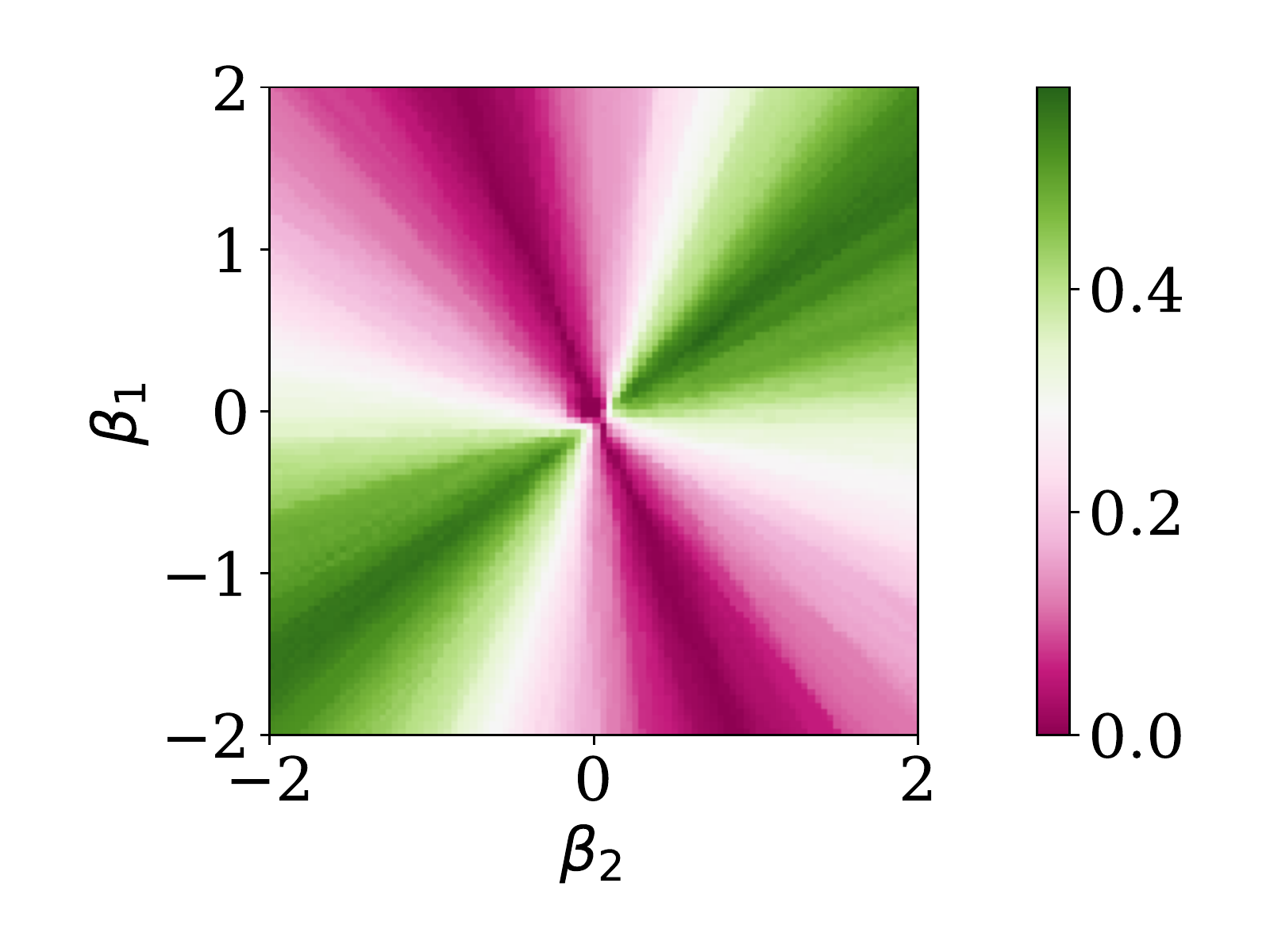}
    \caption{Equal opportunity}
    \label{fig:unfairness_landscape-opp}
    \end{subfigure}\hspace{.6cm}
    \begin{subfigure}[t]{0.375\columnwidth}
    \centering
    \includegraphics[width=0.923\columnwidth]{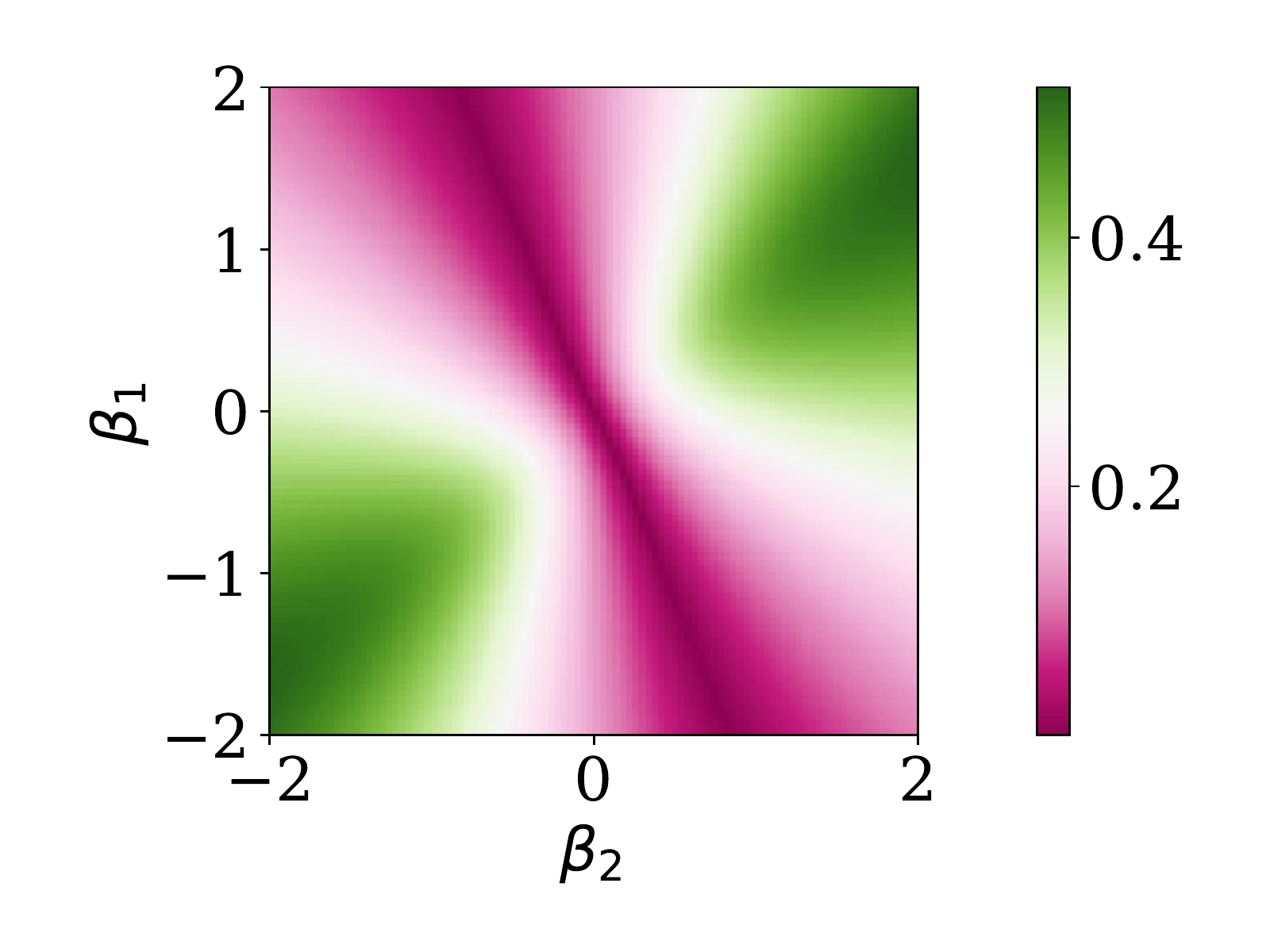}
    \caption{Probabilistic equal opportunity}
    \label{fig:unfairness_landscape-probopp}
    \end{subfigure}
    \vspace{-.1cm}
    \caption{Comparison of fairness notions for $d = 2$ and $h_\beta(x) = 1/(1 + \exp(\frac{1}{3} -\beta_1 x_1 - \beta_2 x_2))$. }
    \label{fig:unfairness_landscape}
    \vspace{-.2cm}
\end{figure}

We use the superscript ``opp'' to emphasize that fairness is measured using the probabilistic equal \textit{opp}ortunity criterion. Consequentially, the set of distributions $\mc F_{h_\beta}^\opp$ that makes the logistic classifier $h_\beta$ fair is
\[
     \mc F_{h_\beta}^\opp = \left\{ 
     \begin{array}{l}
     \QQ \in \mc P \text{ such that }
     \EE_{\QQ}[h_\beta(X) | A=1,Y = 1] = \EE_{\QQ}[h_\beta(X) | A = 0, Y = 1]
    \end{array}
    \right\}.
\]
The statistical hypothesis test to verify whether the classifier $h_\beta$ is fair is formulated with the null and alternative hypotheses
\[
    \mc H_0^\opp: \PP \in \mc F_{h_\beta}^\opp, \quad \mc H_1^\opp: \PP \not\in \mc F_{h_\beta}^\opp.
\]
The remainder of this section unfolds as follows. In Section~\ref{sec:opp-proj}, we delineate the computation of the projection of $\Pnom^N$ onto $\mc F_{h_\beta}^\opp$. Section~\ref{sec:opp-limit} studies the limiting distribution of the test statistic, while Section~\ref{sec:opp-favor} examines the most favorable distribution.

\subsection{Wasserstein Projection}
\label{sec:opp-proj}
Lemma~\ref{lemma:marginal} suggests that it is sufficient to consider the projection onto the marginally-constrained set $\mc F_{h_\beta}^\opp(\wh p^N)$, where $\wh p^N$ is the empirical marginals of the empirical distribution $\Pnom^N$. In particular, $\mc F_{h_\beta}^\opp(\wh p^N)$ is 
\begin{align*}
    \mc F_{h_\beta}^\opp(\wh p^N) = \left\{
     \QQ \in \mc P: \begin{array}{l}
    ( \wh p_{11}^{N})^{-1} \EE_{\QQ}[h_\beta(X) \mathbbm{1}_{(1, 1)}(A, Y)] = (\wh p_{01}^{N})^{-1} \EE_{\QQ}[h_\beta(X) \mathbbm{1}_{(0, 1)}(A, Y)] \\
     \QQ(A = a, Y = y) = \wh p_{ay}^N ~~\forall (a, y) \in \mc A \times \mc Y
    \end{array}
    \right\},
\end{align*}
where the equality follows from the law of conditional expectation. Notice that the set $\mc F_{h_\beta}^\opp(\wh p^N)$ is prescribed using \textit{linear} constraints of $\QQ$, and thus it is more amenable to optimization than the set $\mc F_{h_\beta}^\opp$. It is also more convenient to work with the \textit{squared} distance function $\mc{R}$ whose input is the empirical distribution $\Pnom^N$ and its corresponding vector of empirical marginals $\wh p^N$ by
\begin{align*}
    \mc{R}^\opp(\Pnom^N, \wh p^N) \defeq\left\{
        \begin{array}{cl}
            \inf & \Wass(\QQ, \Pnom^N)^2 \\
            \st & \EE_{\QQ}[h_\beta(X) ( (\wh p_{11}^N)^{-1} \mathbbm{1}_{(1, 1)}(A, Y) - (\wh p_{01}^N)^{-1} \mathbbm{1}_{(0, 1)}(A, Y))] = 0 \\ [1.5ex]
            & \EE_{\QQ}[\mathbbm{1}_{(a, y)}(A, Y)] = \wh p_{ay}^N \quad \forall (a, y) \in \mc A \times \mc Y.
        \end{array}
    \right.
\end{align*}
Notice that the constraints of the above infimum problem are linear in the measure $\QQ$, but the functions inside the expectation operators are possibly \textit{non}linear functions of $\wh p^N$. Using the equivalent characterization~\eqref{eq:relation}, the following relation holds
\[
     \inf_{\QQ \in \mc F_{h_\beta}^\opp}~\Wass(\Pnom^N, \QQ) =  \inf_{\QQ \in \mc F_{h_\beta}^\opp(\wh p^N)}~\Wass(\Pnom^N, \QQ) =  \sqrt{\mc{R}^\opp(\Pnom^N, \wh p^N)}.
\]
We now proceed to show how computing the projection can be reduced to solving a finite-dimensional optimization problem.

\begin{proposition}[Dual reformulation] \label{prop:R-refor}
    The squared projection distance $\mc{R}^\opp(\Pnom^N, \wh p^N)$ equals to the optimal value of the following finite-dimensional optimization problem 
    \begin{align}
    \label{eq:R_opp_refor_finx}
        \Sup{\dualvar \in \R} ~ \frac{1}{N} \sum_{i \in \mc I_1} \Inf{x_i \in \mc X} \left\{ \| x_i - \wh x_i \|^2 + \dualvar \lambda_i h_\beta(x_i) \right\} .
    \end{align}
\end{proposition}

While Proposition~\ref{prop:R-refor} asserts that computing the \textit{squared} projection distance $\mc{R}^\opp(\Pnom^N, \wh p^N)$ is equivalent to solving a finite-dimensional problem, unfortunately, this saddle point problem is in general difficult. Indeed, because $h_\beta$ is non-convex, even finding the optimal inner solution $x_i\opt$ for a fixed value of the outer variable $\gamma \in \R$ is generally NP-hard \cite{ref:murty1985some}. The situation can be partially alleviated if $\| \cdot \|$ is an Euclidean norm on $\R^d$.

\begin{lemma}[Univariate reduction] \label{lemma:R-compute}
    Suppose that $\|\cdot\|$ is the Euclidean norm on $\R^d$, we have
    \begin{align}
        \mc{R}^\opp(\Pnom^N, \wh p^N)=
         \Sup{\dualvar \in \R} { \frac{1}{N}\sum_{i \in \mc I_1} \Min{ k_i \in [0, \frac{1}{8}]}~ \gamma^2\lambda_i^2 \| \beta\|_2^2 k_i^2 + \frac{\gamma \lambda_i}{1 + \exp(\gamma \lambda_i \| \beta\|_2^2 k_i - \beta^\top \wh x_i)}}. \label{eq:R-refor-opp-2}
    \end{align}
\end{lemma}
The proof of Lemma~\ref{lemma:R-compute} follows trivially from application of Lemma~\ref{lemma:individual} to reformulate the inner infimum problems for each $i \in \mc I_1$. Lemma~\ref{lemma:R-compute} offers a significant reduction in the computational complexity to solve the inner subproblems of~\eqref{eq:R_opp_refor_finx}. Instead of optimizing over $d$-dimensional vector $x_i$, the representation in Lemma~\ref{lemma:R-compute} suggests that it suffices to search over a $1$-dimensional space for $k_i$. While the objective function is still non-convex in $k_i$, we can perform a grid search over a compact interval to find the optimal solution for $k_i$ to high precision. The grid search operations can also be parallelized across the index $i$ thanks to the independent structure of the inner problems. Furthermore, the objective function of the supremum problem is a point-wise minimum of linear, thus concave, functions of $\gamma$.
Hence, the outer problem is a concave maximization problem in $\gamma$, which can be solved using a golden section search algorithm.

\subsection{Limiting Distribution}
\label{sec:opp-limit}
We now characterize the limit properties of $\mc{R}^\opp(\Pnom^N, \wh p^N)$. The next theorem assert that the limiting distribution is of the chi-square type.
\begin{theorem}[Limiting distribution -- Probabilistic equal opportunity] \label{thm:limiting-opp}
    Suppose that $(\wh x_i, \wh a_i, \wh y_i)$ are i.i.d.~samples from $\PP$. Under the null hypothesis $\mc H_0^\opp$, we have 
    \begin{align*}
        &N \times \mc{R}^\opp(\Pnom^N, \wh p^N) \xrightarrow{d.} \theta \chi_1^2,
    \end{align*}
    where $\chi_1^2$ is a chi-square distribution with 1 degree of freedom,
    \[
    \theta = \left(\EE_{\PP} \left[\left\| \nabla h_\beta(X) \left( \frac{\mathbbm{1}_{(1, 1)}(A, Y)}{ p_{11}}- \frac{\mathbbm{1}_{(0, 1)}(A, Y)}{p_{01}} \right) \right \|_*^2 \right] \right)^{-1} \frac{\sigma_1^2}{p_{01}^2 p_{11}^2}
    \] 
    with $\sigma_1^2 = \mathrm{Cov}( Z_1)$, and $Z_1$ is the random variable
    \begin{align*}
        Z_1&= h_\beta(X)\left( p_{01}\mathbbm{1}_{(1,1)}( A,Y) 
     -p_{11}\mathbbm{1}_{(0,1)}( A, Y) \right)\\ & \qquad +\mathbbm{1}_{(0,1)}(
    A,Y) \EE_{\PP}[\mathbbm{1}_{(1,1)}( A, Y)
    h_\beta(X)]\\
    &\qquad -\mathbbm{1}_{(1,1)}( A, Y) \EE_{\PP}[\mathbbm{1}_{(0,1)}(
    A, Y) h_\beta(X)].
\end{align*}
\end{theorem}

\noindent \textbf{Construction of the hypothesis test.} Based on the result of Theorem~\ref{thm:limiting-opp}, the statistical hypothesis test proceeds as follows. Let $\eta_{1-\alpha}^\opp$ denote the $(1-\alpha)\times100\%$ quantile of $\theta \chi_1^2$, where $\alpha \in (0, 1)$ is the predetermined significance level. By Theorem~\ref{thm:limiting-opp}, the statistical decision has the form
\begin{center}
    Reject $\mc H_0^\opp$ if $\wh s_N^\opp > \eta_{1-\alpha}^\opp$
\end{center}
with $
    \wh s_N^\opp = N \times \mc{R}^\opp(\Pnom^N, \wh p^N)$.
The limiting distribution $\theta \chi_1^2$ is nonpivotal because $\theta$ depends on the true distribution $\PP$. Luckily, because the quantile function of $\theta \chi_1^2$ is continuous in $\theta$, if $\wh \theta_N$ is a consistent estimator of $\theta$ then it is also valid to use the quantile of $\wh \theta_N \chi_1^2$ for the purpose of testing. We thus proceed to discuss a consistent estimator $\wh \theta_N$ constructed from the available data. First, notice that $\wh p_{01}^N$ and $\wh p_{11}^N$ are consistent estimator for $p_{01}$ and $p_{11}$. Similarly, the law of large numbers asserts that the denominator term in the definition of $\theta$ can be estimated by the sample average
\begin{align*}
    \EE_{\PP} \left[\left\| \nabla h_\beta(X) \left( \frac{\mathbbm{1}_{(1, 1)}(A, Y)}{ p_{11}}- \frac{\mathbbm{1}_{(0, 1)}(A, Y)}{p_{01}} \right)\right \|_*^2 \right]\\
    &\hspace{-2cm}\approx \wh T^N =\frac{\|\beta\|_*^2}{N} \sum_{i=1}^N h_\beta(\wh x_i)^2 (1- h_\beta(\wh x_i))^2\left ( \frac{ \mathbbm{1}_{(1,1)}(\wh a_i, \wh y_i)}{(\wh p_{11}^N)^2} + \frac{ \mathbbm{1}_{(0,1)}(\wh a_i, \wh y_i) }{ (\wh p_{01}^N)^{2} }\right).
\end{align*}
Under the null hypothesis $\mc H_0^\opp$, $Z_1$ has mean 0. The sample average estimate of $\sigma_1^2$ is $\sigma_1^2 \approx (\wh \sigma^N)^2$ with
\begin{align}
 (\wh \sigma_1^N)^2& = \frac{1}{N} \sum_{i=1}^N \Big[ h_\beta(\wh x_i)\left( p_{01}\mathbbm{1}_{(1,1)}( \wh a_i , \wh y_i) 
 -p_{11}\mathbbm{1}_{(0,1)}( A, Y) \right) \notag \\ 
 & \qquad +\mathbbm{1}_{(0,1)}(
\wh a_i, \wh y_i) \big( \sum_{j=1}^N \mathbbm{1}_{(1,1)}( \wh a_j, \wh y_j)
h_\beta(\wh x_j) \big) \label{eq:sigma-estimate} \\
&\qquad -\mathbbm{1}_{(1,1)}( \wh a_i, \wh y_i) \big( \sum_{j=1}^N \mathbbm{1}_{(0,1)}(
\wh a_j, \wh y_j) h_\beta(\wh x_j)\big) \Big]^2. \notag
\end{align}
Using a nested arguments involving the continuous mapping theorem and Slutsky's theorem, the estimator
\[
\wh \theta^N =  \frac{(\wh \sigma_1^N)^2}{\wh T^N (\wh p_{01}^N)^2 (\wh p_{11}^N)^2}
\]
is consistent for $\theta$. Let the corresponding $(1-\alpha)\times100\%$ quantile of the random variable $\wh \theta^N \chi_1^2$ be $\wh \eta_{1-\alpha}^\opp$. The statistical test decision using the plug-in consistent estimate becomes
\begin{center}
    Reject $\mc H_0^\opp$ if $\wh s_N^\opp > \wh \eta_{1-\alpha}^\opp$.
\end{center}

\subsection{Most Favorable Distributions}
\label{sec:opp-favor}
We now discuss the construction of the most favorable distribution $\QQ\opt$, the projection of the empirical distribution $\Pnom^N$ onto the set $\mc F_{h_\beta}^\opp$. Intuitively, $\QQ\opt$ is the distribution closest to $\Pnom^N$ that makes $h_\beta$ a fair classifier under the equal opportunity criterion. If $\|\cdot\|$ is the Euclidean norm, the information about $\QQ\opt$ can be recovered from the optimal solution of problem~\eqref{eq:R-refor-opp-2} by the result of the following lemma.
\begin{lemma}[Most favorable distribution]
\label{lemma:favorable-opp}
Suppose that $\|\cdot\|$ is the Euclidean norm. Let $\dualvar\opt$ be the optimal solution of problem~\eqref{eq:R-refor-opp-2}, and for any $i \in \mc I_1$, let $k_i\opt$ be a solution of the inner minimization of~\eqref{eq:R-refor-opp-2} with respect to $\dualvar\opt$. Then the most favorable distribution $\QQ\opt = \arg\Min{\QQ \in \mc F_{h_\beta}^\opp}~\Wass(\Pnom^N, \QQ)$ is a discrete distribution of the form
\[
\QQ\opt = \frac{1}{N} \Big( \sum_{i\in\mc I_0} \delta_{(\wh x_i, \wh a_i, \wh y_i)} + \sum_{i \in \mc I_1} \delta_{(\wh x_i - k_i\opt \dualvar\opt \lambda_i \beta, \wh a_i, \wh y_i)}\Big).
\]
\end{lemma}

By using the result of Lemma~\ref{lemma:R-compute}, it is easy to verify that $\QQ\opt$ satisfies $\Wass(\QQ\opt, \Pnom^N)^2 = \mc{R}^\opp(\Pnom^N, \wh p^N)$. Moreover, one can also show that $\QQ\opt \in \mc F_{h_\beta}^\opp$. These two observations imply that $\QQ\opt$ is the projection of $\Pnom^N$ onto $\mc F_{h_\beta}^\opp$. The detailed proof is omitted.

Lemma~\ref{lemma:favorable-opp} suggests that in order to obtain the most favorable distribution, it suffices to perturb only the data points with positive label. This is intuitively rational because the notion of probabilistic equality of opportunity only depends on the positive label, and thus the perturbation with a minimal energy requirement should only move sample points with $\wh y_i = 1$. When the underlying geometry is the Euclidean norm, the optimal perturbation of the point $\wh x_i$ is to move it along a line dictated by $\beta$ with a scaling factor $k_i\opt \gamma\opt \lambda_i$. Notice that $\lambda_i$ defined in~\eqref{eq:lambdai-def} are of opposite signs between samples of different sensitive attributes, which implies that it is optimal to perturb $\wh x_i$ in opposite directions dependent on whether $\wh a_i = 0$ or $\wh a_i = 1$. This is, again, rational because moving points in opposite direction brings the clusters of points closer to the others, which reduces the discrepancy in the expected value of $h_\beta(X)$ between subgroups.

As a final remark, we note that $\QQ\opt$ is not necessarily unique. This is because of the non-convexity of the inner problem over $k_i$ in~\eqref{eq:R-refor-opp-2}, which leads to the non-uniqueness of the optimal solution $k_i\opt$ (see Appendix~\ref{sec:app-k-result} and Figure~\ref{fig:non-convex-k}).

\section{Testing Fairness for Probabilistic equalized odds Criterion}
\label{sec:EOdd}
In this section, we extend the Wasserstein projection framework to the statistical test of probabilistic equalized odds for a pre-trained logistic classifier.
\begin{definition}[Probabilistic equalized odds criterion \cite{ref:pleiss2017fairness}]
    A logistic classifier $h_\beta(\cdot) : \mc X \to [0, 1]$ satisfies the probabilistic equalized odds criteria relative to $\QQ$ if
    \[
        \EE_{\QQ}[h_\beta(X) | A = 1, Y = y] = \EE_{\QQ}[h_\beta(X) | A = 0, Y = y] \quad \forall y \in \mc Y.
    \]
\end{definition}
The notion of probabilistic equalized odds requires that the conditional expectation of $h_\beta$ to be independent of $A$ for any label subgroup, thus it is more stringent than the probabilistic equal opportunity studied in the previous section. We use the superscript ``odd'' in this section to emphasize on this specific notion of fairness. The definition of the probabilistic equalized odds prescribes the following set of distributions
\[
     \mc F_{h_\beta}^\odd = \left\{ 
     \QQ \in \mc P:  \begin{array}{l}
     \EE_{\QQ}[h_\beta(X) | A=1,Y = 1]= \EE_{\QQ}[h_\beta(X) | A = 0, Y = 1]\\
     \EE_{\QQ}[h_\beta(X) | A=1,Y = 0] = \EE_{\QQ}[h_\beta(X) | A = 0, Y = 0]
    \end{array}
    \right\}.
\]
Correspondingly, the Wasserstein projection hypothesis test for probabilisitc equalized odds can be formulated as
\[
    \mc H_0^\odd: \PP \in \mc F_{h_\beta}^\odd, \quad \mc H_1^\odd: \PP \not\in \mc F_{h_\beta}^\odd.
\]

In the sequence, we study the projection onto the manifold $\mc F_{h_\beta}^\odd$ in Section~\ref{sec:odd-proj}. Section~\ref{sec:odd-limit} examines the asymptotic behaviour of the test statistic, and we close this section by studying the most favorable distribution $\QQ\opt$ in Section~\ref{sec:odd-favor}.
\subsection{Wasserstein Projection}
\label{sec:odd-proj}

Following a similar strategy as in Section~\ref{sec:EO}, we define the set 
\begin{align*}
    \mc F_{h_\beta}^\odd(\wh p^N) 
    =\left\{ 
     \QQ \in \mc P : \begin{array}{l}
    ( \wh p_{11}^{N})^{-1} \EE_{\QQ}[h_\beta(X) \mathbbm{1}_{(1, 1)}(A, Y)] = (\wh p_{01}^{N})^{-1} \EE_{\QQ}[h_\beta(X) \mathbbm{1}_{(0, 1)}(A, Y)] \\ [1ex]
    ( \wh p_{10}^{N})^{-1} \EE_{\QQ}[h_\beta(X) \mathbbm{1}_{(1, 0)}(A, Y)] = (\wh p_{00}^{N})^{-1} \EE_{\QQ}[h_\beta(X) \mathbbm{1}_{(0, 0)}(A, Y)] \\ [1ex]
     \QQ(A = a, Y = y) = \wh p_{ay}^N ~~\forall (a, y) \in \mc A \times \mc Y
    \end{array}
    \right\},
\end{align*}
and the squared distance function
\begin{align*}
    \mc{R}^\odd(\Pnom^N, \wh p^N) = \left\{
        \begin{array}{cl}
            \inf & \Wass(\QQ, \Pnom^N)^2 \\[.2 cm]
            \st & \EE_{\QQ}[h_\beta(X) ( (\wh p_{11}^N)^{-1} \mathbbm{1}_{(1, 1)}(A, Y) - (\wh p_{01}^N)^{-1} \mathbbm{1}_{(0, 1)}(A, Y))] = 0 \\[.2cm]
            & \EE_{\QQ}[h_\beta(X) ( (\wh p_{10}^N)^{-1} \mathbbm{1}_{(1, 0)}(A, Y) - (\wh  p_{00}^N)^{-1} \mathbbm{1}_{(0, 0)}(A, Y))] = 0 \\[.2cm]
            & \EE_{\QQ}[\mathbbm{1}_{(a, y)}(A, Y)] = \wh p_{ay}^N \quad \forall (a, y) \in \mc A \times \mc Y.
        \end{array}
    \right.
\end{align*}
The equivalent relation~\eqref{eq:relation} suggests that the projection onto the set of distributions $\mc F_{h_\beta}^\odd$ satisfies
\[
    \inf_{\QQ \in \mc F_{h_\beta}^\odd}~\Wass(\Pnom^N, \QQ) =  \inf_{\QQ \in \mc F_{h_\beta}^\odd(\wh p^N)}~\Wass(\Pnom^N, \QQ) = \sqrt{\mc R^\odd(\Pnom^N, \wh p^N)}.
\]

The squared distance $\mc R^\odd(\Pnom^N, \wh p^N)$ can be computed by solving the saddle point problem in the following proposition.

\begin{proposition}[Dual reformulation] \label{prop:R-refor-odd}
    The squared projection distance $\mc{R}^\odd(\Pnom^N, \wh p^N)$ equals to the optimal value of the following finite-dimensional optimization problem 
    \be \label{eq:R-refor-odd}
     \Sup{\dualvar \in \R, \zeta \in \R} ~ \frac{1}{N} \sum_{i=1}^N \Inf{x_i \in \mc X} \left\{ \| x_i - \wh x_i \|^2 +  
     (\dualvar \lambda_i \mathbbm{1}_1 (\wh y_i) +\zeta
     \lambda_i \mathbbm{1}_0 (\wh y_i)) h_\beta(x_i) \right\} .
    \ee
\end{proposition}

To complete this section, we now discuss an efficient way to compute~$\mc{R}^\odd(\Pnom^N, \wh p^N)$. The next lemma reveals that computing $\mc{R}^\odd(\Pnom^N, \wh p^N)$ can be decomposed into two subproblems of similar structure.
    
\begin{lemma}[Univariate reduction] \label{lemma:R-compute-odd}
    We have 
    \[
        \mc{R}^\odd(\Pnom^N, \wh p^N) = \mc{R}^\opp(\Pnom^N, \wh p^N) + U_N,
    \]
    where $U_N$ is computed as
    \begin{align*}
        U_N = \Sup{\zeta \in \R} ~ \frac{1}{N} \sum_{i \in \mc I_0} \Inf{x_i \in \mc X} \left\{ \| x_i - \wh x_i \|^2 + \zeta \lambda_i h_\beta(x_i) \right\} .
    \end{align*}
    Furthermore, if $\|\cdot\|$ is the Euclidean norm on $\R^d$, then
    \begin{align}
        U_N = 
         \Sup{\zeta \in \R} \frac{1}{N}\left\{ \sum_{i \in \mc I_0} \Min{ k_i \in [0, \frac{1}{8}]}~ \zeta^2\lambda_i^2 \| \beta\|_2^2 k_i^2 + \frac{\zeta \lambda_i}{1 + \exp(\zeta \lambda_i \| \beta\|_2^2 k_i - \beta^\top \wh x_i)}
        \right\}. \label{eq:U_N}
    \end{align}
\end{lemma}

Notice that problem~\eqref{eq:U_N} has a similar structure to problem~\eqref{eq:R-refor-opp-2}: the mere difference is that the summation in the objective function of~\eqref{eq:U_N} runs over the index set $\mc I_0 = \{i \in [N]: \wh y_i = 0\}$ instead of $ \mc I_1$ in~\eqref{eq:R-refor-opp-2}. Solving for $U_N$ thus incurs the same computational complexity as, and can also be performed in parallel with, computing $\mc{R}^\opp(\Pnom^N, \wh p^N)$.
\subsection{Limiting Distribution}
\label{sec:odd-limit}

The next result asserts that the squared projection distance $\mc R^\odd$ has the $O(N^{-1})$ convergence rate.
\begin{theorem}[Limiting distribution -- Probabilistic equalized odds] \label{thm:limiting-odd}
    Suppose that $(\wh x_i, \wh a_i, \wh y_i)$ are i.i.d.~samples from $\PP$. Under the null hypothesis $\mc H_0^\odd$, we have
    \begin{align*}
        N \times \mc{R}^\odd(\Pnom^N, \wh p^N) \xrightarrow{d.}  \Sup{\gamma, \zeta
        }\Bigg\{\gamma H_1 + \zeta H_0+  \EE_\PP\Bigg[\Bigg\Vert\begin{pmatrix}
            \gamma \\
            \zeta
        \end{pmatrix}^\top
        \begin{pmatrix} 
        p_{11}^{-1} \mathbbm{1}_{(1, 1)}(A, Y)- p_{01}^{-1} \mathbbm{1}_{(0, 1)}(A, Y) \\
        p_{10}^{-1} \mathbbm{1}_{(1, 0)}(A, Y)- p_{00}^{-1} \mathbbm{1}_{(0, 0)}(A, Y)
        \end{pmatrix} \nabla h_\beta(X) 
        \Bigg\Vert_*^2 \Bigg]\Bigg\},
    \end{align*}
    where $\nabla h_\beta(X) = h_\beta(X) (1- h_\beta(X) \beta$, and $H_y = \mc N(0, \sigma_y^2) /(p_{1y}p_{0y})$ with $\sigma_y^2 = \mathrm{Cov}(Z_y)$, and $Z_y$ are random variables
    \begin{align*}
    Z_y &= h_\beta(X)\left( p_{0y}\mathbbm{1}_{(1,y)}( A,Y) 
 -p_{1y}\mathbbm{1}_{(0,y)}( A, Y) \right)\\ & \qquad +\mathbbm{1}_{(0,y)}(
A,Y) \EE_{\PP}[\mathbbm{1}_{(1,y)}( A, Y)
h_\beta(X)]\\
&\qquad -\mathbbm{1}_{(1,y)}( A, Y) \EE_{\PP}[\mathbbm{1}_{(0,y)}(
A, Y) h_\beta(X)].
\end{align*}
\end{theorem}

\noindent \textbf{Construction of the hypothesis test.} Contrary to the explicit chi-square limiting distribution for the probabilistic equal opportunity fairness in Theorem~\ref{thm:limiting-opp}, the limiting distribution for the probabilistic equalized odds fairness is not available in closed form. Nevertheless, the limiting distribution in this case can be obtained by sampling $H_0$ and $H_1$ and solving a collection of optimization problems for each sample. Notice that the objective function of the supremum problem presented in Theorem~\ref{thm:limiting-odd} is continuous in $H_1$ and $H_0$, one thus can define
\[
    \wh H_y = \mc N(0, \wh \sigma_y^2)/ (\wh p_{1y}^N \wh p_{0y}^N),
\]
where $\wh \sigma_y^2$ is the sample average estimate of $\sigma_y^2$, which can be computed using an equation similar to~\eqref{eq:sigma-estimate}. The limiting distribution can be computed by solving the optimization problem with plug-in values
\begin{align*}
    \Sup{\gamma, \zeta
    }\Bigg\{\gamma \wh H_1 + \zeta \wh H_0+ ~\EE_{\Pnom^N}\Bigg[\Bigg\Vert\begin{pmatrix}
        \gamma \\[.2cm]
        \zeta
    \end{pmatrix}^\top
    \begin{pmatrix} 
    (\wh p_{11}^N)^{-1} \mathbbm{1}_{(1, 1)}(A, Y)- (\wh p_{01}^N)^{-1} \mathbbm{1}_{(0, 1)}(A, Y) \\[.2cm]
    (\wh p_{10}^N)^{-1} \mathbbm{1}_{(1, 0)}(A, Y)- (\wh p_{00}^N)^{-1} \mathbbm{1}_{(0, 0)}(A, Y)
    \end{pmatrix} \nabla h_\beta(X) 
    \Bigg\Vert_*^2 \Bigg]\Bigg\}.
\end{align*}
Notice that the expectation in taken over the empirical distribution $\Pnom^N$, and can be written as a finite sum. The last optimization problem can be solved efficiently using quadratic programming for any realization of $\wh H_1$ and $\wh H_0$. The objective values can be collected to compute the $(1-\alpha) \times 100\%$-quantile estimate $\wh \eta_{1-\alpha}^\odd$ of the limiting distribution. The statistical test decision using the plug-in estimate becomes
\begin{center}
    Reject $\mc H_0^\odd$ if $\wh s_N^\odd > \wh \eta_{1-\alpha}^\odd$,
\end{center}
where $\wh s_N^\odd = N \times \mc R^\odd(\Pnom^N, \wh p^N)$.
\subsection{Most Favorable Distributions}
\label{sec:odd-favor}

If the feature space $\mc X$ is endowed with an Euclidean norm, then the most favorable distribution $\QQ\opt$, defined in this section as the projection of $\Pnom^N$ onto $\mc F_{h_\beta}^\odd$, can be constructed by exploiting Lemma~\ref{lemma:R-compute-odd}.

\begin{lemma}[Most favorable distribution]
\label{lemma:favorable-odd}
Suppose that $\|\cdot\|$ is the Euclidean norm. Let $\dualvar\opt$ and $\zeta\opt$ be the optimal solution of problems~\eqref{eq:R-refor-opp-2} and~\eqref{eq:U_N}, respectively. For any $i \in \mc I_1$, let $k_i\opt$ be the solution of the inner minimization of~\eqref{eq:R-refor-opp-2} with respect to $\dualvar\opt$, and for any $i \in \mc I_0$, let $k_i\opt$ be a solution of the inner minimization of~\eqref{eq:U_N} with respect to $\zeta\opt$. Then the most favorable distribution $\QQ\opt = \arg\min_{\QQ \in \mc F_{h_\beta}^\odd}~\Wass(\Pnom^N, \QQ)$ is a discrete distribution of the form
\[
\QQ\opt = \frac{1}{N} \Big( \sum_{i\in\mc I_0} \delta_{(\wh x_i - k_i\opt \zeta\opt \lambda_i \beta, \wh a_i, \wh y_i)} + \sum_{i \in \mc I_1} \delta_{(\wh x_i - k_i\opt \dualvar\opt \lambda_i \beta, \wh a_i, \wh y_i)}\Big).
\]
\end{lemma}

The proof of Lemma~\ref{lemma:favorable-odd} follows from verifying that $\QQ\opt \in \mc F_{h_\beta}^\odd$ and that $\Wass(\QQ\opt, \Pnom^N)^2 = \mc{R}^\odd(\Pnom^N, \wh p^N)$ using Lemma~\ref{lemma:R-compute-odd}, the detailed proof is omitted. For probabilistic equalized odds, the most favorable distribution $\QQ\opt$ alters the locations of both $i \in \mc I_0$ and $i \in \mc I_1$. The directions of perturbation are dependent on $\lambda_i$, which is determined using~\eqref{eq:lambdai-def}. Notice that $\lambda_i$ carry opposite signs corresponding to whether $\wh a_i = 0$ or $\wh a_i = 1$, thus the perturbations will move $\wh x_i$ in opposite directions based on the value of the sensitive attribute $\wh a_i$.

\section{Numerical Experiment}
\label{sec:experiment}
All experiments are run on an Intel Xeon based cluster composed of 287 compute nodes each with 2 Skylake processors running at 2.3 GHz with 18 cores each. We only use 2 nodes of this cluster and all optimization problems are implemented in Python version 3.7.3.
In all experiments, we use the 2-norm to measure distances in the feature space. Moreover, we focus on the hypothesis test of probabilistic equal opportunity, and thus the Wasserstein projection, the limiting distribution and the most favorable distribution follow from the results presented in Section~\ref{sec:EO}.
\subsection{Validation of the Hypothesis Test}
We now demonstrate that our proposed Wasserstein projection framework for statistical test of fairness is a valid, or asymptotically correct, test. We consider a binary classification setting in which $\mc X$ is 2-dimensional feature space. The true distribution $\PP$ has true marginal values $p_{ay}$ being
\[
p_{11} = 0.2,~p_{01} = 0.1,~p_{10} = 0.3,~p_{00} = 0.4.
\]
Moreover, conditioning on $(A, Y)$, the feature $X$ follows a Gaussian distribution of the form
\begin{align*}
X|A = 1, Y = 1 &\sim \mc N([6, 0], [3.5, 0; 0, 5]),\\
X|A = 0, Y = 1 &\sim \mc N ([-2, 0], [5, 0; 0, 5]),\\
X|A = 1, Y = 0 &\sim \mc N ([6, 0], [3.5, 0; 0, 5]),\\
X|A = 0, Y = 0 &\sim \mc N ([-4, 0], [5, 0; 0, 5]).
\end{align*}
The true distribution $\PP$ is thus a mixture of Gaussian, and under this specification, a simple algebraic calculation indicates that a logistic classifier with $\beta = (0, 1)^\top$ is fair with respect to the probabilistic equal opportunity criteria in Definition~\ref{def:proba-EO}. We thus focus on verifying fairness for this specific classifier.
    \begin{figure*}[h]
        \centering
        \begin{subfigure}[h]{0.37\textwidth}
        \includegraphics[width=\textwidth]{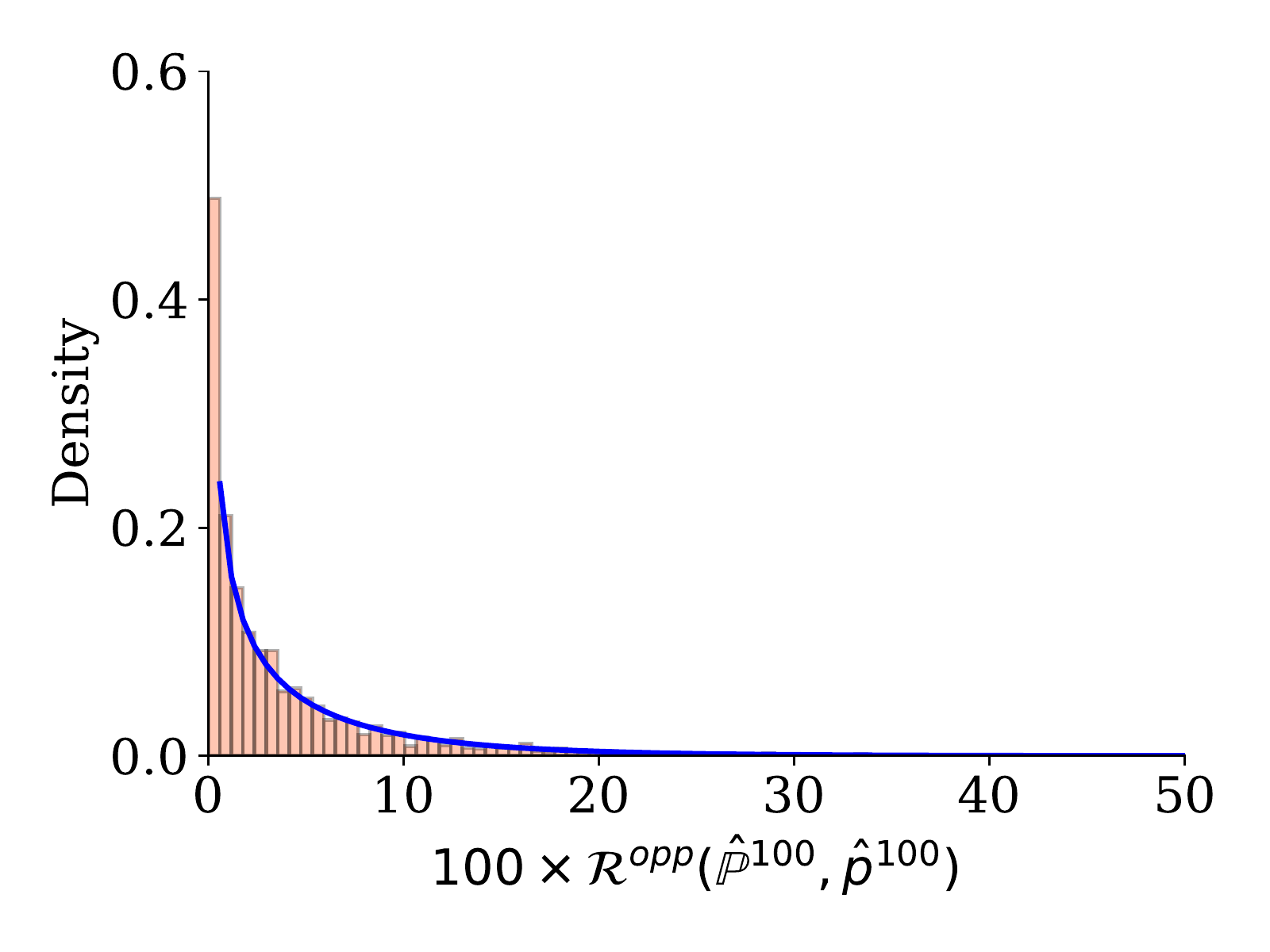}
        \caption{$N=100$}
        \label{fig:limita}
        \end{subfigure}
        \begin{subfigure}[h]{0.37\textwidth}
        \includegraphics[width=\textwidth]{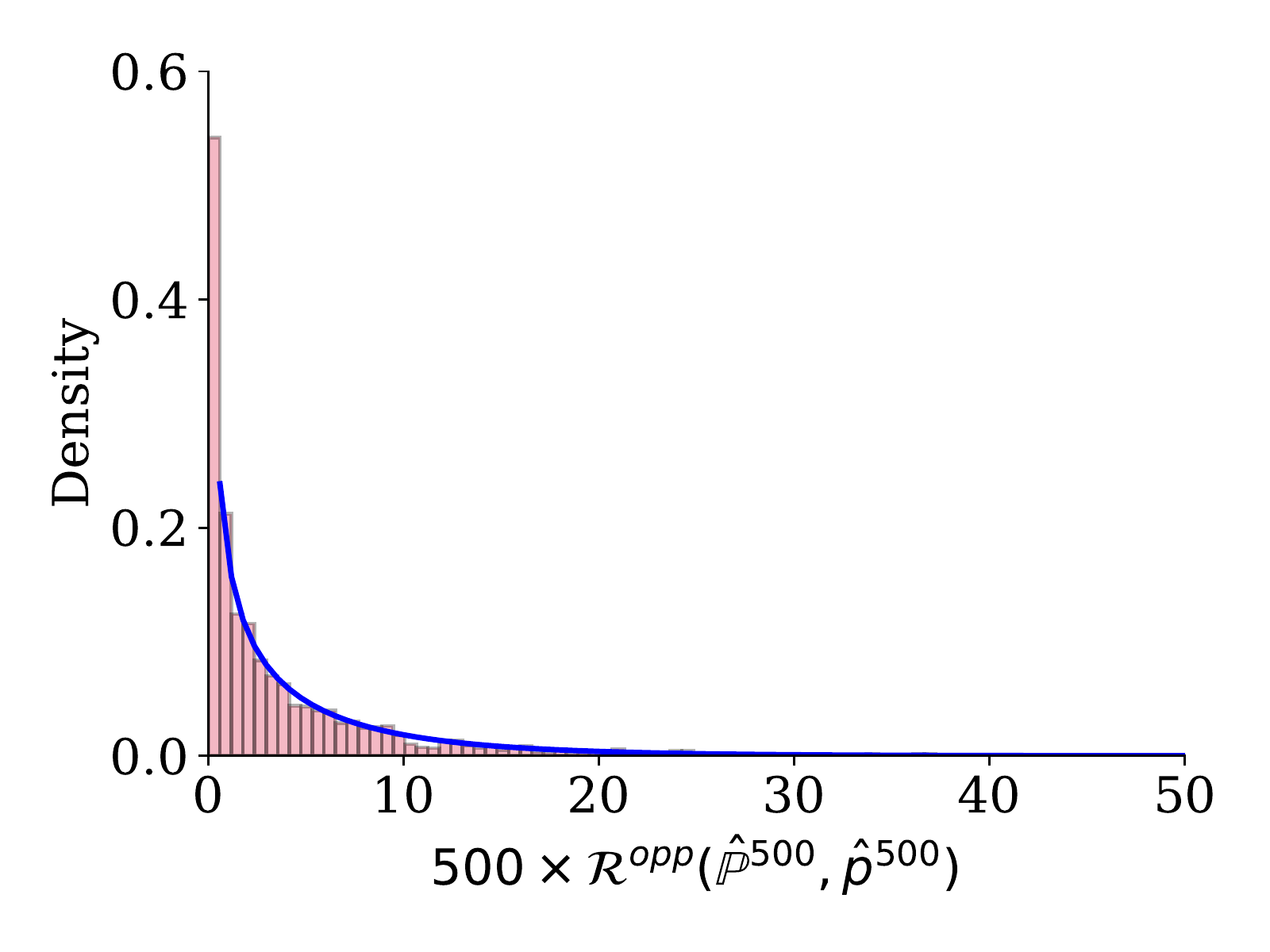}
        \caption{$N=500$}
        \label{fig:limitb}
        \end{subfigure}
        \begin{subfigure}[h]{0.37\textwidth}
        \includegraphics[width=\textwidth]{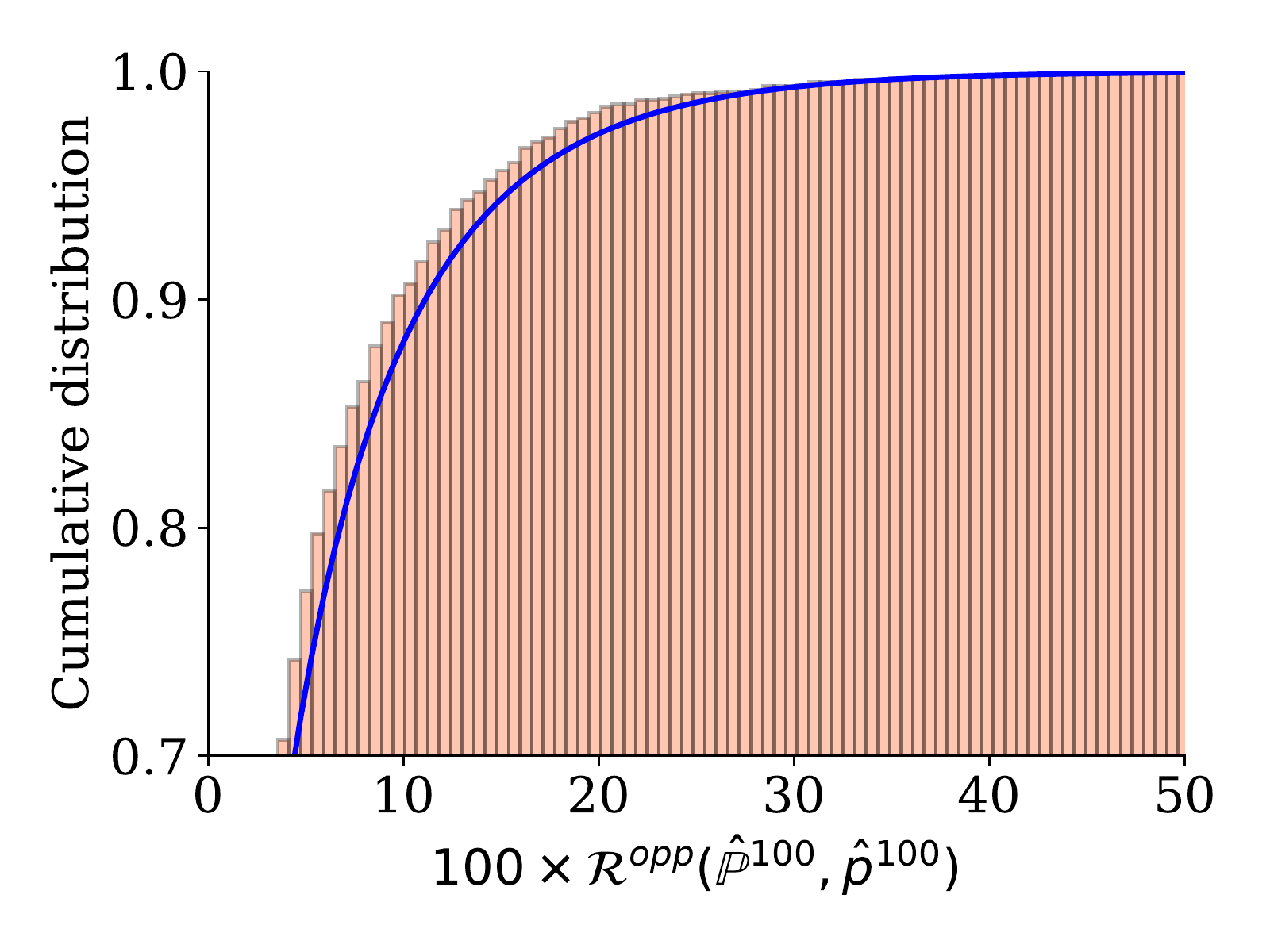}
        \caption{$N=100$}
        \label{fig:limitc}
        \end{subfigure}
        \begin{subfigure}[h]{0.37\textwidth}
        \includegraphics[width=\textwidth]{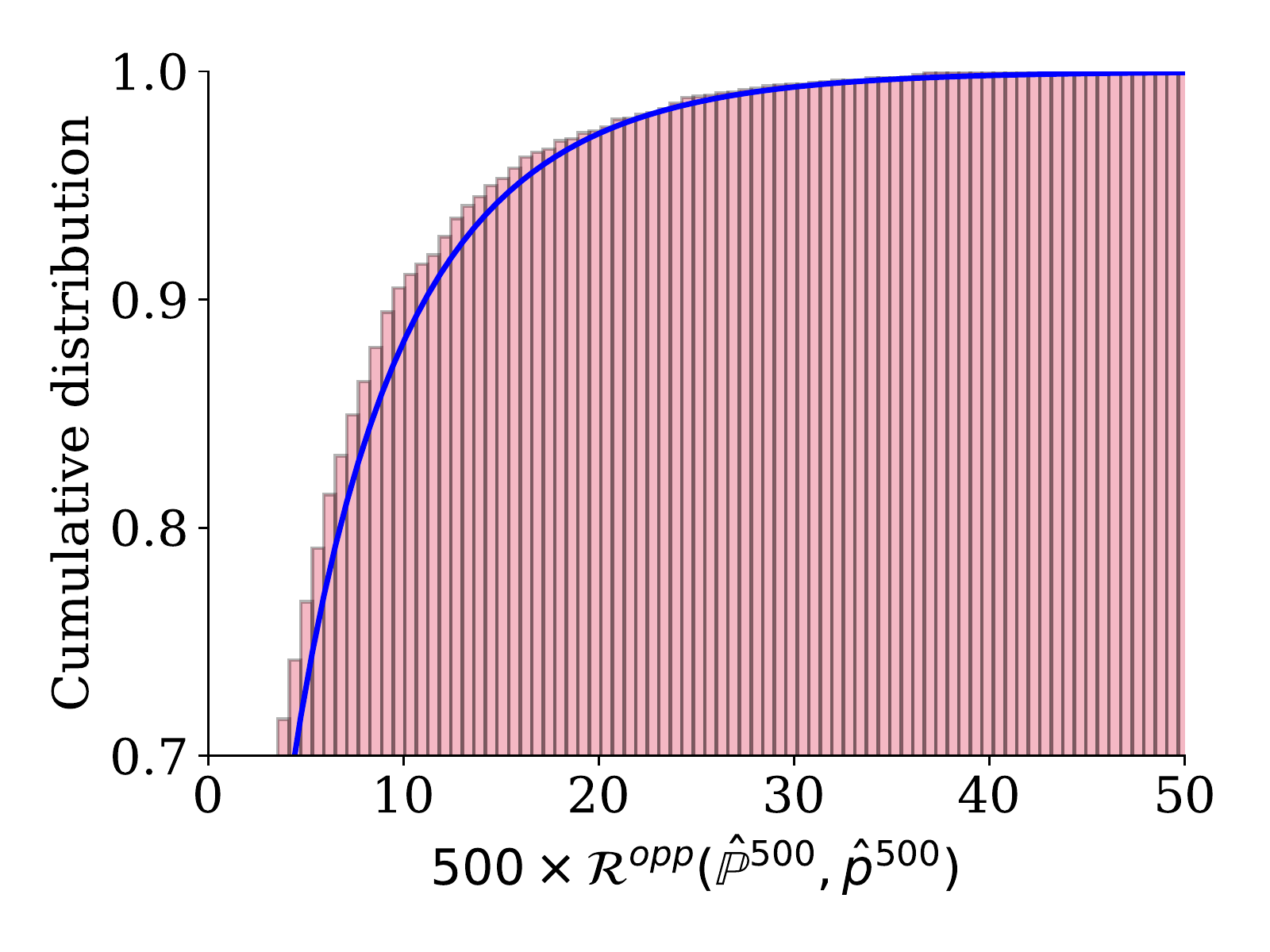}
        \caption{$N=500$}
        \label{fig:limitd}
        \end{subfigure}
        \caption{Empirical distribution of $N \times \mc R^\opp(\Pnom^N, \wh p^N)$ taken over 2,000 replications (histogram) versus the limiting distribution $\theta \chi_1^2$ (blue curve) with different sample sizes $N$. Fig.~\ref{fig:limita}-\ref{fig:limitb} are density plots, Fig.~\ref{fig:limitc}-\ref{fig:limitd} are cumulative distribution plots. }
        \label{fig:limit}
    \end{figure*}
In the first experiment, we empirically validate Theorem~\ref{thm:limiting-opp}. To this end, we generate $N \in \{100, 500\}$ i.i.d.~samples from $\PP$ to be used as the test data, and then calculate the squared projection distance $\mc R^\opp(\Pnom^N, \wh p^N)$ using Proposition~\ref{prop:R-refor}. The process is repeated 2,000 times to obtain an empirical estimate of the distribution of $N \times \mc R^\opp(\Pnom^N, \wh p^N)$. We also generate another set of one million i.i.d.~samples from $\PP$ to estimate the limiting distribution $\theta \chi_1^2$. Figure~\ref{fig:limit} shows that the empirical distribution of $N \times \mc R^\opp(\Pnom^N, \wh p^N)$ converges to the limiting distribution $\theta \chi_1^2$ as $N$ increases.

The second set of experiments aims to show that our proposed Wasserstein projection hypothesis test is asymptotically valid. We generate  $N \in \{100, 500, 1000\}$ i.i.d.~samples from $\PP$ and calculate the test statistic $N \times \mc R^\opp(\Pnom^N, \wh p^N)$. The same data is used to estimate $\wh \theta^N$ and compute the $(1-\alpha)\times 100\%$-quantile of $\wh \theta^N \chi_1^2$ to perform the quantile based test as laid out in Section~\ref{sec:opp-limit}. We repeat this procedure for 2,000 replications to keep track of the rejection projection at different significant values of $\alpha \in \{0.5, 0.3, 0.1, 0.05, 0.01\}$. 
Table~\ref{tab:rej_rates_table} summarizes the rejection probabilities of Wasserstein projection tests for equal opportunity criterion under the null hypothesis $\mc H_0^\opp$. We can observe that at sample size $N > 100$, the rejection probability is close to the desired level $\alpha$, which empirically validates our testing procedure.
\begin{center}
\vspace{-.2cm}
\begin{table}[h]
 \begin{tabular}{||c c c c ||c||}
 \hline \hline
 $N=100$ & $N=500$  & $N=1000 $ & $\alpha$\\
 \hline\hline
  0.511 &0.4905 & 0.5 & 0.50\\ \hline
  0.282 &0.2895 &0.299 & 0.30\\\hline
  0.048 &0.0895 &0.093 & 0.10\\ \hline
  0.007 & 0.0425 & 0.0405 &0.05\\\hline
  0.0 & 0.0065 & 0.005 & 0.01\\
 \hline \hline
\end{tabular}
\vspace{.2cm}
 \caption{Comparison of the null rejection probabilities of probabilistic equal opportunity tests with different significance levels~$\alpha$ and test sample sizes $N$.} \label{tab:rej_rates_table}
\end{table}
\end{center}
\subsection{Most Favorable Distribution Analysis}
\begin{figure}
    \centering
    \includegraphics[width=0.4\columnwidth]{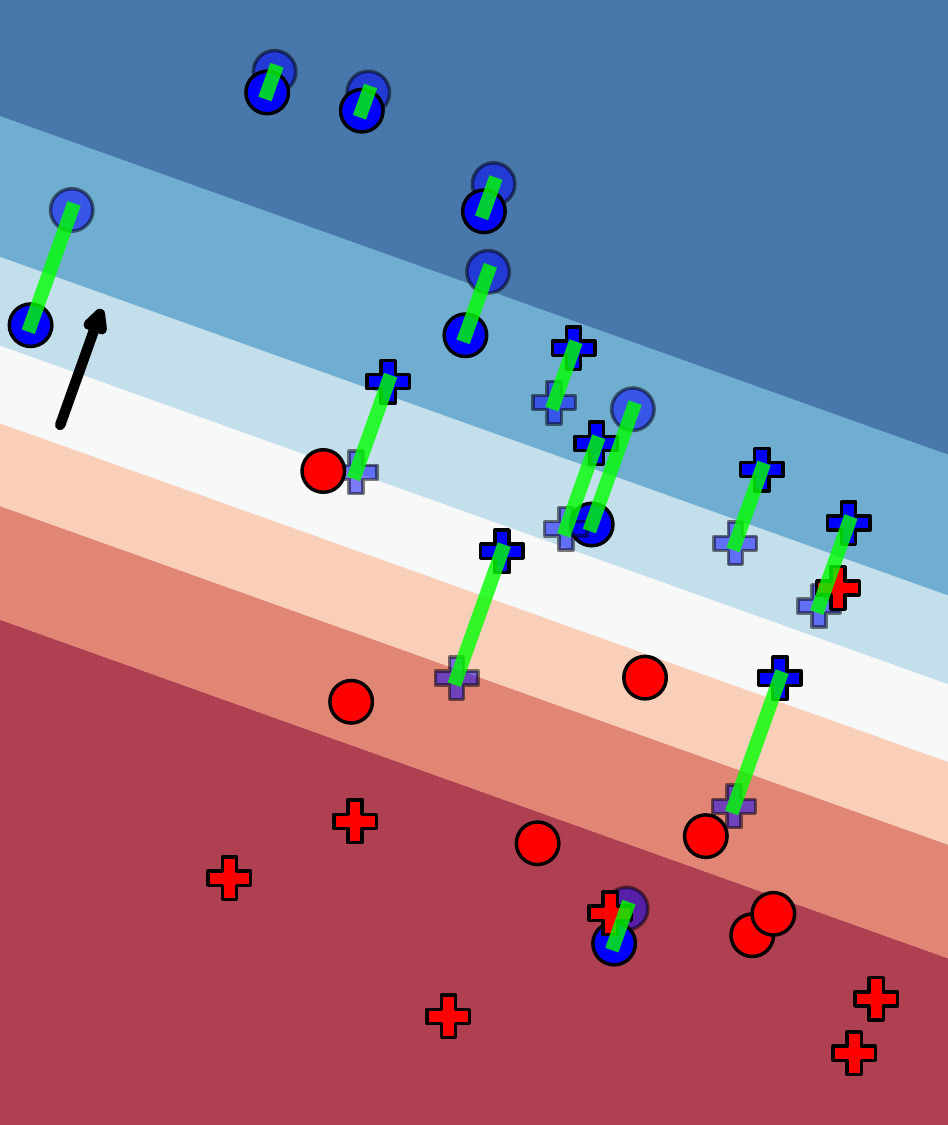}
    \caption{Visualization of the most favorable distribution $\QQ\opt$ for a logistic classifier with weight $\beta =(0.4, 0.12)^\top$. The black arrow indicates the vector $\beta$. Colors represent class, while symbolic shapes encode the sensitive values.
    The green lines show the transport plan of the empirical test samples from their original positions (indicated with transparent colors) to their ultimate destinations (with non-transparent colors).}
    \label{fig:most_fav}
\end{figure}

In this section, we visualize the most favorable distribution $\QQ\opt$ from Lemma~\ref{lemma:favorable-opp} for a vanilla logistic regression classifier with weight $\beta=(0.4, 0.12)^\top$. We simply generate 28 samples with equal subgroup proportions to form the empirical distribution $\Pnom^N$. To find the support of~$\QQ\opt$, we solve problem~\eqref{eq:R-refor-opp-2}, whose optimizer dictates the transportation plan of each sample $\wh x_i$. Figure~\ref{fig:most_fav} visualizes the original test samples that forms $\Pnom^N$, along with the most favorable distribution $\QQ\opt$. Green lines in the figure represent how samples are perturbed. As we are testing for the probabilistic notion of equal opportunity, only the samples with positive label $\wh y_i=1$ presented in blue are perturbed in order to obtain $\QQ\opt$. Furthermore, we observe that the positively-labeled test samples are transported along the axis directed by $\beta$ (black arrow). Moreover, the samples with different sensitive attributes, represented by different shapes, move in opposite direction so that they get closer to each other, which reduces the discrepancy in the expected value of $h_\beta(X)$ between the relevant subgroups.

\subsection{The COMPAS Dataset}
\begin{figure}
    \centering
    \includegraphics[width=0.6\columnwidth]{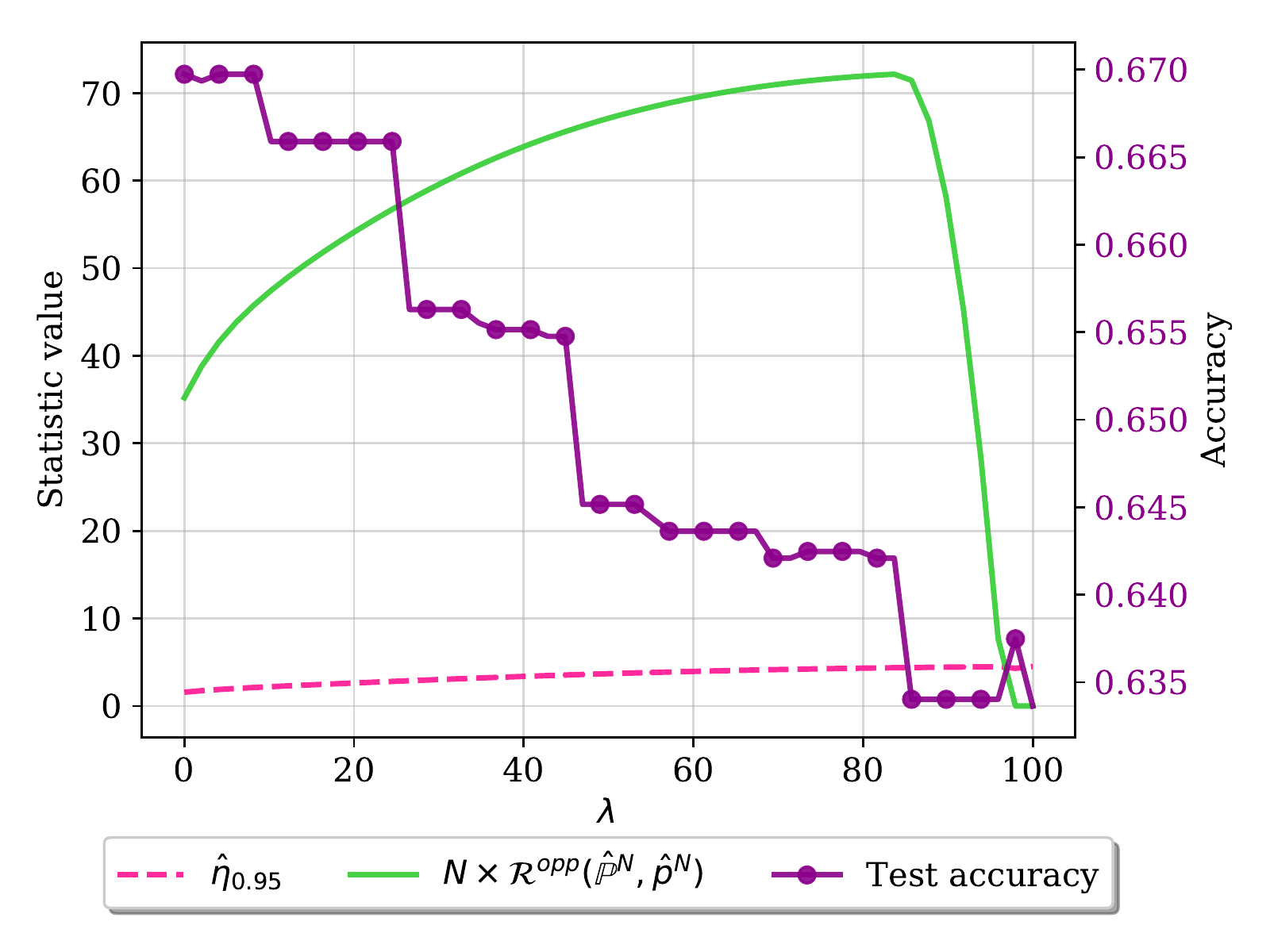}
    \caption{Test statistic and accuracy of Tikhonov regularized logistic regression on test data with rejection threshold $\wh \eta_{0.95}$.}
    \vspace{-5mm}
    \label{fig:reg_figure}
\end{figure}

COMPAS (Correctional Offender Management Profiling for Alternative Sanctions)\footnote{https://www.propublica.org/datastore/dataset/compas-recidivism-risk-score-data-and-analysis} is a commercial tool used by judges and parole officers for scoring criminal defendant’s likelihood of recidivism. The COMPAS dataset is used by the COMPAS algorithm to compute the risk score of reoffending for defendants, and also contains the criminal records within 2 years after the decision. The dataset consists of 6,172 samples with 10 attributes including gender, age category, race, etc. We concentrate on the subset of the data with violent recidivism, and we use race (African-American and Caucasian) as the sensitive attribute. We split 70$\%$ of the COMPAS data to train a Tikhonov-regularized logistic classifier, with the tuning penalty parameter $\lambda$ chosen in the range from 0 to 100 with 50 equi-distant points. The remaining $30\%$ of the data is used as the test samples for auditing.

Figure~\ref{fig:reg_figure} demonstrates the relation between the accuracy and the degree of fairness with respect to the regularization parameter $\lambda$. Strong regularization penalty (high values of $\lambda$) results in small values of the test statistic, but the classifier has low test accuracy. On the contrary, weak penalization leads to undesirable fairness level but higher prediction accuracy.
The pink dashed line in Figure~\ref{fig:reg_figure} shows the rejection threshold of the Wasserstein projection test at significance level $\alpha = 0.05$ for varying value of the regularization parameter $\lambda$. We can observe that the Wasserstein projection test recommends a rejection of the null hypothesis $\mc H_0^\opp$ for a wide range of $\lambda$. Only at $\lambda$ sufficiently large that the test fails to reject the null hypothesis.

\section{Concluding Remarks and Broader Impact} 
\label{sec:conclusion}

In this paper, we propose a statistical hypothesis test for group fairness of classification algorithms based on the theory of optimal transport. Our test statistic relies on computing the projection distance from the empirical distribution supported on the test samples to the manifold of distributions that renders the classifier fair. When the notion of fairness is chosen to be either the probabilistic equal opportunity or the probabilistic equalized odds, we show that the projection can be computed efficiently. We provide the limiting distribution of the test statistic and show that our Wasserstein projection test is asymptotically correct. Our proposed test also offers the flexibility to incorporate the geometric information of the feature space into testing procedure. Finally, analyzing the most favorable distribution can help interpreting the reasons behind the outcome of the test.

The Wasserstein projection hypothesis test is the culmination of a benevolent motivation and effort, and it aims to furnish the developers, the regulators and the general public a quantitative method to verify certain notions of fairness in the classification setting. At the same time, we acknowledge the risks and limitations of the results presented in this paper.

First, it is essential to keep in mind that this paper focuses on \textit{probabilistic} notions of fairness, in particular, we provide the Wasserstein statistical test for probabilistic equality of opportunity and probabilistic equalized odds. Probabilistic notions are only approximations of the original definitions, and the employment of probabilistic notions are solely for the technical purposes. Due to the sensitivity of the test result on the choice of fairness notions, a test that is designed for probabilistic notions may not be applicable to test for original notions of fairness due to the interplay with the threshold $\tau$ and the radical difference of both the test statistic and the limiting distribution. If a logistic classifier $h_\beta$ is rejected using our framework for probabilistic equal opportunity, it does \textit{not} necessarily imply that the classifier $h_\beta$ fails to satisfy the equal opportunity criterion, and vice versa. The same argument holds when we test for probabilistic equalized odds.

Second, the outcome of the Wasserstein projection test is dependent on the choice of the underlying metric on the feature, the sensitive attribute and the label spaces. Indeed, the test outcome can change if we switch the metric of the feature space, for example, from the Euclidean norm to a 1-norm. In the scope of this paper, we do not study how sensitive the test outcome is with respect to the choice of the metric, nor can we make any recommendation on the optimal choice of the metric. Nevertheless, it is reasonable to recommend that the metric should be chosen judiciously, and the action of tuning the metric in order to obtain favorable test outcome should be prohibited.

Third, to simplify the computation, we have assumed absolute trust on the sensitive attributes and the label. The users of our test should be mindful if there is potential corruption to these values. Moreover, our test is constructed under the assumption that there is no missing values in the test data. This assumption, unfortunately, may not hold in real-world implementations. Constructing statistical test which is robust to adversarial attacks and missing data using the Wasserstein projection framework is an interesting research direction.

Fourth, the statistical test in this paper is for a simple null hypothesis. In practice, the regulators may be interested in a relaxed fairness test in which the difference of the conditional expectations is upper bounded by a fixed positive constant $\epsilon$. The extension of the Wasserstein hypothesis testing framework for a composite null hypothesis is non-trivial, thus we leave this idea for future study.

Finally, any auditing process for algorithmic fairness can become a dangerous tool if it falls into the hand of unqualified or vicious inspectors. The results in this paper are developed to broaden our scientific understanding, and we recommend that the test and its outcomes should be used as an informative reference, but \textit{not} as an absolute certification to promote any particular classifier or as a justification for any particular classification decision.

We thus sincerely recommend that the tools proposed in this paper be exercised with utmost consideration.

\textbf{Acknowledgments.}
Research supported by the Swiss National Science Foundation under NCCR Automation, grant agreement 51NF40$\_$180545. Material in this paper is based upon work supported by the Air Force Office of Scientific Research under award number FA9550-20-1-0397. Additional support is gratefully acknowledged from NSF grants 1915967, 1820942, 1838676, and also from the China Merchant Bank. 
Finally, we would like to thank Nian Si and Michael Sklar for helpful comments and discussions.

\appendix
\section{Appendix - Proofs}

\subsection{Proofs of Section~\ref{sec:framework}}
\begin{proof}[Proof of Lemma~\ref{lemma:marginal}]
Because the fairness constraints are similar in both sets $\mc F_h$ and $\mc F_h(\wh p^N)$, it thus suffice to verify that $\QQ$ satisfies the marginal conditions $\QQ(A = a, Y = y) = \wh p_{ay}^N$ for all $(a, y) \in \mc A \times \mc Y$. By the definition of the Wasserstein distance and the ground metric $c$, there exists a coupling $\pi$ such that
\[
\Wass(\Pnom^N, \QQ)^2 = \EE_\pi[ (\| X' - X\| + \infty | A' - A| + \infty | Y' - Y|)^2]
\]
and the marginal distribution of $\pi$ are $\Pnom^N$ and $\QQ$, respectively. By the law of total probability and because $\Pnom^N$ is an empirical distribution, we can write $\pi = {N}^{-1}\sum_{i=1}^N \delta_{(\hat x_i , \hat a_i, \hat y_i)}\otimes \QQ_i$, where $\QQ_i$ denotes the conditional distributions of $(X, A, Y)$ given $(X', A', Y') = (\hat x_i , \hat a_i , \hat y_i)$ for all $i \in [N]$.

Suppose without any loss of generality that there exists a tuple $(a, y) \in \mc A \times \mc Y$ such that $\QQ(A = a, Y = y) > \wh p_{ay}^N$. This means
\begin{align*}
    \QQ(A = a, Y = y) &= \frac{1}{N}\sum_{i=1}^N \QQ_i(A = a, Y = y)> \frac{1}{N} \sum_{i=1}^N \mathbbm{1}_{(a, y)}(\wh a_i, \wh y_i).
\end{align*}
This implies that there must exist an index $i\opt \in [N]$ with $(\wh a_{i\opt}, \wh y_{i\opt}) \neq (a, y)$, and that
\[
    \QQ_{i\opt}(A = a, Y = y) > 0.
\]
However, this further implies that
\begin{align*}
    \Wass(\Pnom^N, \QQ)^2 &= \frac{1}{N} \sum_{i=1}^N \EE_{\QQ_i}[ (\| \wh x_i - X\| + \infty | \wh a_i - A| + \infty | \wh y_i - Y|)^2] \\
    &\ge \frac{1}{N} \EE_{\QQ_{i\opt}}[ (\| \wh x_{i\opt} - X\| + \infty | \wh a_{i\opt} - A| + \infty | \wh y_{i\opt} - Y|)^2] \\
    &\ge \frac{1}{N} \QQ_{i\opt}(A = a, Y = y) \left( \infty (\wh a_{i\opt} - a) + \infty (\wh y_{i\opt} - y) \right)^2 =\infty,
\end{align*}
where the equality follows from the decomposition of $\pi$ using the law of total probability and the first inequality follows because the transportation cost is nonnegative. This contradicts the fact that $\Wass(\Pnom^N, \QQ) < \infty$.
\end{proof}

\subsection{Proofs of Section~\ref{sec:EO}}

Before proving Proposition~\ref{prop:R-refor}, we first prove a preparatory lemma that verifies the Slater condition of the conic optimization problem. To shorten the notation, we write $\xi = (X, A, Y)$ and denote $\Xi = \mc X \times \mc A \times \mc Y$, $\wh \Xi_N = \{(\wh x_i, \wh a_i, \wh y_i)\}_{i=1}^N$. We assume that $N \ge 2$ and $\wh \xi_i = (\wh x_i, \wh a_i, \wh y_i)$ are distinct.  We use $\mc M_+(\Xi \times \wh \Xi_N)$ to denote the set of all nonnegative measures on $\Xi \times \wh \Xi_N$.
\begin{lemma}[Slater condition - Probabilistic equal opportunity] \label{lemma:slater-opp}
    Suppose that $\beta \neq 0$, $\wh p_{11}^N \in (0,1)$ and $\wh p_{01}^N \in (0, 1)$. Define the function 
    \[f_\beta(X, A, Y)\Let\frac{1}{\wh p_{11}^N} h_\beta(X) \mathbbm{1}_{(1,1)}(A, Y) -  \frac{1}{\wh p_{01}^N} h_\beta(X)  \mathbbm{1}_{(0,1)}(A, Y), \]
    and let $f$ be a vector-valued function $f: \Xi \times \wh \Xi_N \to \R^{N+1}$ 
    \[
        f(\xi, \xi') = \begin{pmatrix}
         \mathbbm{1}_{\hat \xi_i}(\xi') \\
         \vdots \\
         \mathbbm{1}_{\hat \xi_N}(\xi') \\
         f_\beta(\xi) 
    \end{pmatrix}.
    \]
    Then we have
    \[
        \begin{pmatrix}
            1/N \\
            \vdots \\
            1/N \\
            0
        \end{pmatrix} \in
        \mathrm{int}\left\{ 
            \EE_{\pi}[f(\xi, \xi')]: \pi \in \mc M_+(\Xi \times \wh \Xi_N)
        \right\}.
    \]
\end{lemma}
\begin{proof}[Proof of Lemma~\ref{lemma:slater-opp}]
    It suffices to show that for any 
    \[
        q \in \left(\frac{1}{2N}, \frac{3}{2N} \right)^N \times \left(-\frac{1}{4}, \frac{1}{4} \right),
    \]
    there exists a nonnegative measure $\pi \in \mc M_+(\Xi \times \wh \Xi_N)$ such that $q = \EE_{\pi}[f(\xi, \xi')]$. We will verify this claim by constructing $\pi$ explicitly. To this end, define the following locations
    \[
        x_{ay} \in \mc X \quad \forall (a, y) \in \mc A \times \mc Y,
    \]
    and set $\pi \in \mc M_+(\Xi \times \wh \Xi_N)$ explicitly as
    \[
        \pi( \xi = (x_{\wh a_i \wh y_i}, \wh a_i, \wh y_i), \xi' = (\wh x_i, \wh a_i, \wh y_i) ) = q_i,
    \]
    and $\pi$ is 0 everywhere else. By construction, one can verify that $\EE_{\pi}[\mathbbm{1}_{\wh \xi_i}(\xi')] = q_i$ for all $i \in [N]$. If we define the following index sets $\mc I_{ay} = \{ i \in [N]: \wh a_i = a, \wh y_i = y\}$, then
    \begin{align*}
        \EE_{\pi}[f_\beta(\xi)] = (\wh p_{11}^N)^{-1} h_\beta(x_{11}) \sum_{i \in \mc I_{11}} q_i - (\wh p_{01}^N)^{-1} h_\beta(x_{01}) \sum_{i \in \mc I_{01}} q_i.
    \end{align*}
    It now remains to find the locations of $x_{11}$ and $x_{01}$ to balance the above equation. We have the following two cases.
    \begin{enumerate}[leftmargin = 5mm]
        \item Suppose that $q_{N+1} \ge 0$. In this case, choose $x_{01} \in \mc X$ such that $h_\beta(x_{01}) = \frac{1}{6}$. The condition $\EE_{\pi}[f_\beta(\xi)] = q_{N+1}$ requires that
        \[
            h_\beta(x_{11}) = \ds \frac{q_{N+1} + \frac{1}{6} (\wh p_{01}^N)^{-1} \sum_{i \in \mc I_{01}}q_i}{(\wh p_{11}^N)^{-1} \sum_{i \in \mc I_{11}} q_i}.
        \]
        Because $q_{N+1} \ge 0$ and $q_i$ are strictly positive, the term on the right hand side is strictly positive. Moreover, we have
        \begin{align*}
            (\wh p_{01}^N)^{-1} \sum_{i \in \mc I_{01}}q_i < \frac{3}{2} \quad \text{and} \quad
            (\wh p_{11}^N)^{-1} \sum_{i \in \mc I_{11}} q_i  > \frac{1}{2}
        \end{align*}
        for any feasible value of $q_i$, which implies that
        \[
            0 < \ds \frac{q_{N+1} + \frac{1}{6} (\wh p_{01}^N)^{-1} \sum_{i \in \mc I_{01}}q_i}{(\wh p_{11}^N)^{-1} \sum_{i \in \mc I_{11}} q_i} < \frac{\frac{1}{4} + \frac{1}{4}}{\frac{1}{2}} = 1.
        \]
        This implies the existence of $x_{11} \in \mc X$ so that $ \EE_{\pi}[f_\beta(\xi)] = q_{N+1}$.
        \item Suppose that $q_{N+1} < 0$. In this case, we can choose $x_{11} \in \mc X$ such that $h_\beta(x_{11}) = \frac{1}{6}$. A similar argument as in the previous case implies the existence of $x_{01} \in \mc X$ such that $ \EE_{\pi}[f_\beta(\xi)] = q_{N+1}$.
    \end{enumerate}
    Combining the two cases leads to the postulated results.
\end{proof}

We are now ready to prove Proposition~\ref{prop:R-refor}.
\begin{proof}[Proof of Proposition~\ref{prop:R-refor}]
    For the purpose of this proof, we define the function $\lambda: \mc A \times \mc Y \to \R$ as
\be \label{eq:lambda-def}
    \lambda(a, y) = \frac{\mathbbm{1}_{(1, 1)}(a, y)}{\wh p_{11}^N} - \frac{\mathbbm{1}_{(0, 1)}(a, y)}{\wh p_{01}^N}.
\ee
By definition of the squared distance function $\mc{R}^\opp$, we have
    \begin{align*}
       \mc{R}^\opp(\Pnom^N, \wh p^N) 
        =& \left\{ \begin{array}{cl}
        \Inf{\QQ \in \mc P}&\Wass(\Pnom^N, \QQ)^2 \\
        \st & (\hat p_{11}^N)^{-1} \EE_{\QQ}[h_\beta(X) \mathbbm{1}_{(1,1)}(A, Y)]  =  (\wh p_{01}^N)^{-1} \EE_{\QQ}[h_\beta(X)  \mathbbm{1}_{(0,1)}(A, Y)] \\
     & \QQ(A = a, Y = y) = \wh p_{ay}^N
        \quad\forall a \in \mc A,~ y\in \mc Y
        \end{array}
        \right. \\
        =& \left\{ \begin{array}{cll}
        \Inf{\pi}& \EE_{\pi}[ c\big( (X', A', Y'), (X, A, Y) \big)^2] \\
        \st & \pi \in \mc P((\mc X \times \mc A \times \mc Y) \times (\mc X \times \mc A \times \mc Y)) \\
        &\EE_{\pi}[f_\beta(X, A, Y)] = 0 \\
        & \pi(A = a, Y = y) = \wh p_{ay}^N
        &\hspace{-.5cm}\forall a \in \mc A,~ y\in \mc Y \\
        &\EE_{\pi}[ \mathbbm{1}_{(\wh x_i, \wh a_i, \wh y_i)}(X', A', Y')] = 1/N &\hspace{-.5cm}\forall i \in [N],
        \end{array}
        \right.
    \end{align*}
    where the function $f_\beta$ is defined as
 \begin{align} 
    f_\beta(x, a, y) &\Let (\wh p_{11}^N)^{-1} h_\beta(x) \mathbbm{1}_{(1,1)}(a, y) -  (\wh p_{01}^N)^{-1} h_\beta(x)  \mathbbm{1}_{(0,1)}(a, y) h_\beta(x) \lambda(a, y), \label{eq:f-def}
\end{align}
    and $\mc P(\mc S)$ denotes the set of all joint probability measures supported on $\mc S$.
    Because of the infinity individual cost on $\mc A$ and $\mc Y$ by the definition of cost in~\eqref{eq:cost}, any joint measure $\pi$ with finite objective value should satisfies $\pi (A= a, Y= y)  = \Pnom^N(A' = a, Y' = y) = \hat p_{ay}^N$ for any $a \in \mc A$ and $y \in \mc Y$. Thus, the set of constraints $\pi(A = a, Y = y) = \hat p_{ay}^N$ can be eliminated without alternating the optimization problem. We thus have
    \begin{align*}
        \mc{R}^\opp(\Pnom^N, \wh p^N) = 
        \left\{ \begin{array}{cl}
        \Inf{\pi}& \EE_{\pi}[ c\big( (X', A', Y'), (X, A, Y) \big)^2 ] \\
        \st & \pi \in \mc P((\mc X \times \mc A \times \mc Y) \times (\mc X \times \mc A \times \mc Y)) \\
        &\EE_{\pi}[f_\beta(X, A, Y)] = 0 \\
        &\EE_{\pi}[\mathbbm{1}_{(\wh x_i, \wh a_i, \wh y_i)}(X', A', Y')] = 1/N \quad \forall i \in [N].
        \end{array}
        \right.
    \end{align*}
    To shorten the notations, we use $\Xi = \mc X \times \mc A \times \mc Y$ and $\wh \Xi_N = \{(\wh x_i, \wh a_i, \wh y_i)\}$. Moreover, define the vector $\bar q$ and the vector-valued Borel measurable function on $\Xi \times \wh \Xi_N$ as \[\bar q= \begin{pmatrix}
         0 \\
         1/N \\
         \vdots \\
         1/N
    \end{pmatrix}~ \hspace{.5cm}
    f(\xi, \xi') = \begin{pmatrix}
         f_\beta(\xi) \\
         \mathbbm{1}_{\hat \xi_i}(\xi') \\
         \vdots \\
         \mathbbm{1}_{\hat \xi_N}(\xi')
    \end{pmatrix}.\]
    By using the introduced notation, we can reformulate the above optimization problem as
    \[\inf \left\{
         \EE_{\pi}[ c(\xi, \xi')^2]:\pi \in \mc M_+(\Xi \times \wh \Xi_N),
        \EE_{\pi}[f(\xi, \xi')] = \bar q \right\}
    \] which is a problem of moments. 
    By Lemma~\ref{lemma:slater-opp}, the above optimization problem satisfies the Slater condition, thus the strong duality result~\cite[Section~2.2]{ref:smith1995generalized} implies that 
    \begin{align}
    \label{eq:s_N_opp_refor}
        \mc{R}^\opp(\Pnom^N, \wh p^N) = \left\{ \begin{array}{cl}
            \sup & \ds   \frac{1}{N} \sum_{i=1}^N b_i \\
            \st & b \in \R^{N},~\dualvar \in \R \\
            &\ds \sum_{i=1}^N b_i \mathbbm{1}_{(\wh x_i, \wh a_i, \wh y_i)}(x', a', y') - \dualvar f_\beta(x, a, y) \le c\big( (x', a', y'), (x, a, y) \big)^2 \\
            &\hspace{1cm}\quad \forall (x, a, y), (x', a', y') \in \mc X \times \mc A \times \mc Y.
        \end{array} 
        \right. 
        \end{align}
         Note that the problem in~\eqref{eq:s_N_opp_refor} can be equivalently represented as
        \begin{align}
        & \left\{ \begin{array}{cl}
            \sup & \ds \frac{1}{N} \sum_{i=1}^N  b_i\\
            \st & b \in \R^{N},~\dualvar \in \R \\
            &\ds b_i - \dualvar f_\beta(x_i, a_i, y_i) \le c\big( (\wh x_i, \wh a_i, \wh y_i), (x_i, a_i, y_i) \big)^2 \\
            &\hspace{2cm}\forall (x_i, a_i, y_i) \in \mc X \times \mc A \times \mc Y, \forall i \in [N]
        \end{array}
        \right. \notag\\
        =& \Sup{\dualvar \in \R} ~ \frac{1}{N} \sum_{i=1}^N \Inf{x_i \in \mc X} \left\{ \| x_i - \wh x_i \|^2 + \dualvar f_\beta(x_i, \wh a_i, \wh y_i) \right\}. \label{eq:good-exp}
    \end{align}
    Because $f_\beta$ has the form~\eqref{eq:f-def}, we have the equivalent problem
    \[
     \Sup{\dualvar \in \R} ~ \frac{1}{N} \sum_{i=1}^N \Inf{x_i \in \mc X} \left\{ \| x_i - \wh x_i \|^2 + \dualvar \lambda(\wh a_i, \wh y_i) h_\beta(x_i) \right\} .
     \]
     For any $i \in \mc I_0$, $\lambda(\wh a_i, \wh y_i) = 0$, and in this case we have the optimal solution of $x_i$ satisfies $x_i\opt = \wh x_i$. As a consequence, the summation collapses to a partial sum over $\mc I_1$. This observation completes the proof.
\end{proof}

\begin{proof}[Proof of Theorem~\ref{thm:limiting-opp}]
Leveraging equation~\eqref{eq:good-exp}, we can express
\begin{align*}
    \mc{R}^\opp(\Pnom^N, \wh p^N) =
    \Sup{\gamma} \EE_{\Pnom^N}\left[\Inf{\Delta}\gamma h_\beta(X + \Delta) \left( \frac{ \mathbbm{1}_{(1,1)}(A, Y) }{\wh p_{11}^N} - \frac{ \mathbbm{1}_{(0,1)}(A, Y)}{\wh p_{01}^N} \right) + \|\Delta\|^2 \right].
\end{align*}
We define 
\begin{align*}
    H^N &\Let \frac{1}{\sqrt{N}} \sum_{i=1}^N h_\beta(\wh x_i) \left( \frac{ \mathbbm{1}_{(1,1)}(\wh a_i, \wh y_i) }{\wh p_{11}^N} - \frac{ \mathbbm{1}_{(0,1)}(\wh a_i, \wh y_i)}{\wh p_{01}^N} \right),
\end{align*}
and using this expression we can reformulate $\mc{R}(\Pnom^N, \wh p^N)$ as
\begin{align}
    \Sup{\gamma} \Bigg\{ \frac{1}{\sqrt{N}}\gamma H^N + \EE_{\Pnom^N}\Big[&\Inf{\Delta} \gamma [h_\beta(X + \Delta) - h_\beta(X)] \times
    \left( \frac{ \mathbbm{1}_{(1,1)}(A, Y) }{\wh p_{11}^N} - \frac{ \mathbbm{1}_{(0,1)}(A, Y)}{\wh p_{01}^N} \right)  +
    \| \Delta\|^2 \Big]\Bigg\}. \notag
\end{align}
Because $h_\beta$ is a sigmoid function, it is differentiable, and by the fundamental theorem of calculus, we have for any $x \in \mc X$, 
\[
    h_\beta(x + \Delta) - h_\beta(x) = \int_0^1 \nabla h_\beta(x+t\Delta) \cdot \Delta \mathrm{d}t,
\]
where $\cdot$ represents the inner product on $\R^d$. By applying variable transformations $\gamma \leftarrow \gamma \sqrt{N}$ and $\Delta \leftarrow \Delta \sqrt{N}$, we have
\begin{align*}
    & N \times \mc{R}^\opp(\Pnom^N, \wh p^N)\\
    =& 
    \Sup{\gamma} \Bigg\{ \gamma H^N + \EE_{\Pnom^N}\Big[\Inf{\Delta} \gamma \int_0^1 \nabla h_\beta\left(X+t\frac{\Delta}{\sqrt{N}}\right) \cdot \Delta \mathrm{d}t \times   \left( \frac{ \mathbbm{1}_{(1,1)}(A, Y) }{\wh p_{11}^N} - \frac{ \mathbbm{1}_{(0,1)}(A, Y)}{\wh p_{01}^N} \right)  + \| \Delta\|^2 \Big]\Bigg\} \\
    = &
    \Sup{\gamma} \Bigg\{ \gamma H^N + \frac{1}{N} \sum_{i=1}^N \Inf{\Delta_i} \gamma \int_0^1 \nabla h_\beta\left(\wh x_i + t\frac{\Delta_i}{\sqrt{N}}\right) \cdot \Delta_i \mathrm{d}t \times\left( \frac{ \mathbbm{1}_{(1,1)}(\wh a_i, \wh y_i) }{\wh p_{11}^N} - \frac{ \mathbbm{1}_{(0,1)}(\wh a_i, \wh y_i)}{\wh p_{01}^N} \right)  + \| \Delta_i\|^2 \Bigg\},
\end{align*}
where the second equality follows by the definition of the empirical distribution $\Pnom^N$. 
For any values of $\wh p_{01}^N > 0$ and $\wh p_{11}^N > 0$, we have for any $\gamma \neq 0$, 
\begin{align*}
    &\PP\!\left(\! \left\| \gamma \nabla h_\beta(X)\left( \frac{ \mathbbm{1}_{(1,1)}(A, \!Y) }{\wh p_{11}^N}\! -\! \frac{ \mathbbm{1}_{(0,1)}\!(A,\! Y)}{\wh p_{01}^N} \right)\!\right\|_*\!\! =\! 0\right) \!=\! \PP\big( (\hat p_{11}^N)^{-1}  \mathbbm{1}_{(1,1)}\!(A, \!Y) \!\!=\!\! (\hat p_{01}^N)^{-1}  \mathbbm{1}_{(0,1)}\!(A, Y) \big) = \PP( Y \!=\! 0 ) \!<\! 1,
\end{align*}
which implies that 
\[
    \PP\left( \left\| \gamma \nabla h_\beta(X)\left( \frac{ \mathbbm{1}_{(1,1)}(A, Y) }{\wh p_{11}^N} - \frac{ \mathbbm{1}_{(0,1)}(A, Y)}{\wh p_{01}^N} \right)\right\|_* > 0\right) > 0.
\]
This coincides with Assumption A4 in~\cite{ref:blanchet2019rwpi}. Using the same argument as in the proof of \cite[Theorem~3]{ref:blanchet2019rwpi}, we can show that the optimal solution for $\gamma$ and $\Delta_i$ belong to a compact set with high probability.
Moreover, we have
\begin{align*}
    \frac{ \mathbbm{1}_{(1,1)}(\wh a_i, \wh y_i) }{\wh p_{11}^N} - \frac{ \mathbbm{1}_{(0,1)}(\wh a_i, \wh y_i)}{\wh p_{01}^N} = 
    \frac{ \mathbbm{1}_{(1,1)}(\wh a_i, \wh y_i) }{ p_{11}} \left(1 - o_{\mathbb{P}}(1)  \right) - \frac{ \mathbbm{1}_{(0,1)}(\wh a_i, \wh y_i)}{ p_{01}} \left(1 - o_{\mathbb{P}}(1)\right),
\end{align*}
and thus
\begin{align*}
    &N \times \mc{R}^\opp(\Pnom^N, \wh p^N) \\
    =& 
    \Sup{\gamma} \Bigg\{ \gamma H^N + \frac{1}{N} \sum_{i=1}^N \Inf{\Delta_i} \gamma \int_0^1 \nabla h_\beta\left(\wh x_i + t\frac{\Delta_i}{\sqrt{N}}\right) \cdot \Delta_i \mathrm{d}t \times\left( \frac{ \mathbbm{1}_{(1,1)}(\wh a_i, \wh y_i) }{ p_{11}} - \frac{ \mathbbm{1}_{(0,1)}(\wh a_i, \wh y_i)}{p_{01}} \right)  + \| \Delta_i\|^2  + o_{\PP}(1)\Bigg\}.
\end{align*}
In the next step, fix any tuple $(a, y) \in \mc A \times \mc Y$, and denote the following constant \[
            M_1 = | p_{11}^{-1}\mathbbm{1}_{(1, 1)}(a, y) -  p_{01}^{-1} \mathbbm{1}_{(0,1)}(a, y)|.
        \]
        We find
        \begin{align*}
            &\| [\nabla h_\beta(x+\Delta) - \nabla h_\beta(x)] (p_{11}^{-1}\mathbbm{1}_{(1, 1)}(a, y) -  p_{01}^{-1} \mathbbm{1}_{(0,1)}(a, y))\|_* \\
            =&| h_\beta(x+\Delta) - h_\beta(x) - h_\beta(x+\Delta)^2 + h_\beta(x)^2| \| \beta\|_* M_1 \\
            \le & (|h_\beta(x + \Delta) - h_\beta(x)| + |h_\beta(x + \Delta)^2 - h_\beta(x)^2| )\| \beta\|_* M_1.
        \end{align*}
        Because the sigmoid function is slope-restricted in the interval $[0, 1]$~\cite[Proposition~2]{ref:fazlyab2019safety}, we have
        \[
            0 \le \frac{h_\beta(x+\Delta) - h_\beta(x)}{\beta^\top \Delta} \le 1,
        \]
        which implies that
        \[
            |h_\beta(x + \Delta) - h_\beta(x)| \leq |\beta^\top \Delta| \leq \|\beta\|_* \|\Delta\|,
        \]
        where the second inequality follows from H\"{o}lder inequality. 
        Using a similar argument, we have
        \begin{align*} 
             |h_\beta(x + \Delta)^2 - h_\beta(x)^2|
             =&\le (h_\beta(x + \Delta)+ h_\beta(x)) |h_\beta(x + \Delta) - h_\beta(x)| \le 2 \|\beta\|_* \|\Delta\|.
        \end{align*}
        Combining these inequalities, we conclude that
        \begin{align*}
        &\| [\nabla h_\beta(x+\Delta) - \nabla h_\beta(x)] (p_{11}^{-1}\mathbbm{1}_{(1, 1)}(a, y) -  p_{01}^{-1} \mathbbm{1}_{(0,1)}(a, y))\|_2  \le 3\|\beta\|_*^2 M_1 \|\Delta\|,
        \end{align*}
        and thus Assumption 6' in~\cite{ref:blanchet2019rwpi} is satisfied. If $H^N \xrightarrow{d.} \tilde Z$ for some random variable $\tilde Z$, then \cite[Lemma~4]{ref:blanchet2019rwpi} asserts that
\begin{align}
\label{eq:limiting_dist_opp_fin}
        &N \times \mc{R}^\opp(\Pnom^N, \wh p^N) \nonumber \\ 
        \xrightarrow{d.}
        &\Sup{\gamma \in \R} \left\{\gamma \tilde Z - \frac{\gamma^2}{4} \EE_{\PP} \left[\left\| \nabla h_\beta(X) \left( \frac{\mathbbm{1}_{(1, 1)}(A, Y)}{ p_{11}} - \frac{\mathbbm{1}_{(0, 1)}(A, Y)}{p_{01}} \right) \right\|_*^2\right] \right\} \nonumber \\ 
        =& \left(\EE_{\PP}\left [\left \| \nabla h_\beta(X) \left( \frac{\mathbbm{1}_{(1, 1)}(A, Y)}{ p_{11}}- \frac{\mathbbm{1}_{(0, 1)}(A, Y)}{p_{01}} \right) \right\|_*^2 \right] \right)^{-1} \tilde Z^2, \notag
\end{align}
where the equality sign follows from the fact that for any realization of $\tilde Z$, the optimal solution of $\gamma$ is 
\[
    \gamma\opt(\tilde Z) = \frac{2 \tilde Z}{\EE_{\PP}\left [\left\| \nabla h_\beta(X) \big( \frac{\mathbbm{1}_{(1, 1)}(A, Y)}{ p_{11}}- \frac{\mathbbm{1}_{(0, 1)}(A, Y)}{p_{01}} \big) \right\|_*^2\right]}.
\]
We now study the limit distribution $\tilde Z$. In the next step, we study the limit of $H^N$. \begin{align*}
H^{N} 
&=\frac{1}{\sqrt{N}}\sum_{i=1}^{N}h_\beta(\wh x_{i})\left( \frac{\mathbbm{1}_{(1,1)}\left(
\wh a_{i}, \wh y_{i}\right) }{\wh p_{11}^N}-\frac{\mathbbm{1}_{(0,1)}\left( \wh a_{i}, \wh y_{i}\right) 
}{\hat{p}_{01}^N}\right) \\
&=\frac{1}{\hat{p}_{11}^N\hat{p}_{01}^N} \times\frac{1}{\sqrt{N}}\sum_{i=1}^{N}h_\beta(\wh x_{i})\left( \wh p
_{01}^N \mathbbm{1}_{(1,1)}\left( \wh a_{i}, \wh y_{i}\right) -\wh{p}_{11}^N \mathbbm{1}_{(0,1)}\left(
\wh a_{i}, \wh y_{i}\right) \right)\\
&= \frac{1}{\hat{p}_{11}^N\hat{p}_{01}^N} \times \Big( \frac{1}{\sqrt{N}}\sum_{i=1}^{N}h_\beta(\wh x_{i})\big( p
_{01} \mathbbm{1}_{(1,1)}\left( \wh a_{i}, \wh y_{i}\right) -{p}_{11}\mathbbm{1}_{(0,1)}(
\wh a_{i}, \wh y_{i}) \big) \\
& \hspace{2cm} + \sqrt{N} ( \hat{p}_{01}^N-p_{01}) \frac{1}{N}%
\sum_{i=1}^{N}\mathbbm{1}_{(1,1)}( \wh a_{i}, \wh y_{i}) h_\beta(\wh x_{i}) \\
& \hspace{2cm} - \sqrt{N} ( \hat{p}_{11}^N-p_{11}) \frac{1}{N}%
\sum_{i=1}^{N}\mathbbm{1}_{(0,1)}( \wh a_{i}, \wh y_{i}) h_\beta(\wh x_{i})  \Big)
\end{align*}
By Slutsky's theorem, we have
\begin{align*}
&\sqrt{N}( \hat{p}_{01}^N-p_{01}) \times
\frac{1}{N}\sum_{i=1}^{N}\left(
\mathbbm{1}_{(1,1)}\left( \wh a_{i}, \wh y_{i}\right) h_\beta(\wh x_{i})-\EE_{\PP}[\mathbbm{1}_{(1,1)}(
A, Y) h_\beta(X)]\right) =o_{\PP}(1), \\
&\sqrt{N}( \hat{p}_{11}^N-p_{11}) \times \frac{1}{N} \sum_{i=1}^{N}\left(
\mathbbm{1}_{(0,1)}( \wh a_{i},\wh y_{i}) h_\beta(\wh x_{i})-\EE_{\PP}[1_{(0,1)}\left(
A,Y\right) h_\beta(X)]\right) =o_{\PP}(1).
\end{align*}
Under the null hypothesis $\mc H_0^\opp$, we have
\begin{align*}
H^{N}
&= \frac{1}{\hat{p}_{11}^N\hat{p}_{01}^N} \times \Big[ \frac{1}{\sqrt{N}}\sum_{i=1}^{N}h_\beta(\wh x_{i})\left( p
_{01} \mathbbm{1}_{(1,1)}\left( \wh a_{i}, \wh y_{i}\right) -{p}_{11}\mathbbm{1}_{(0,1)}(
\wh a_{i}, \wh y_{i}) \right) \\
& \quad + \sqrt{N} \left(\frac{1}{N} \sum_{i=1}^{N}\mathbbm{1}_{(0,1)}( \wh a_{i}, \wh y_{i})-p_{01}\right) \EE_{\PP}[\mathbbm{1}_{(1,1)}(
A, Y) h_\beta(X)] \\
& \quad - \sqrt{N} \left( \frac{1}{N}%
\sum_{i=1}^{N}\mathbbm{1}_{(1,1)}( \wh a_{i}, \wh y_{i})-p_{11}\right)  \EE_{\PP}[\mathbbm{1}_{(0,1)}(
A, Y) h_\beta(X)] \Big] + o_{\PP}(1)\\
&= \frac{1}{\hat{p}_{11}^N\hat{p}_{01}^N} \times \Big[ \frac{1}{\sqrt{N}}\sum_{i=1}^{N}h_\beta(\wh x_{i})\left( p
_{01} \mathbbm{1}_{(1,1)}\left( \wh a_{i}, \wh y_{i}\right) -{p}_{11}\mathbbm{1}_{(0,1)}(
\wh a_{i}, \wh y_{i}) \right) \\
& \quad + \frac{1}{\sqrt{N}} \sum_{i=1}^{N}\left(\mathbbm{1}_{(0,1)}( \wh a_{i}, \wh y_{i})-p_{01}\right) \EE_{\PP}[\mathbbm{1}_{(1,1)}(
A, Y) h_\beta(X)] \\
& \quad - \frac{1}{\sqrt{N}} \sum_{i=1}^{N}\left(\mathbbm{1}_{(1,1)}( \wh a_{i}, \wh y_{i})-p_{11}\right)  \EE_{\PP}[\mathbbm{1}_{(0,1)}(
A, Y) h_\beta(X)] \Big] + o_{\PP}(1) \\
& \xrightarrow{d.} \tilde Z,
\end{align*}
where 
$\tilde Z\sim  \frac{1}{p_{11}p_{01}} \mc N(0, \sigma^2 )$, $\sigma^2 = \mathrm{Cov}(Z)$, where $Z$ is defined as in the theorem statement. Defining $\theta$ completes the proof.
\end{proof}

\subsection{Proofs of Section~\ref{sec:EOdd}}

The proof of Proposition~\ref{prop:R-refor-odd} necessitates the following preparatory lemma. We use the same notations with Lemma~\ref{lemma:slater-opp}.
\begin{lemma}[Slater condition - Probabilistic equalized odds] \label{lemma:slater-odd}
    Suppose that $\beta \neq 0$ and $\wh p_{ay}^N \in (0,1)$ for all $(a, y) \in \mc A \times \mc Y$. Define the functions
    \begin{align*}
        f_\beta(X, A, Y)&\Let\frac{1}{\wh p_{11}^N} h_\beta(X) \mathbbm{1}_{(1,1)}(A, Y) -  \frac{1}{\wh p_{01}^N} h_\beta(X)  \mathbbm{1}_{(0,1)}(A, Y),\\
        g_\beta(X, A, Y)&\Let\frac{1}{\wh p_{10}^N} h_\beta(X) \mathbbm{1}_{(1,0)}(A, Y) -  \frac{1}{\wh p_{00}^N} h_\beta(X)  \mathbbm{1}_{(0,0)}(A, Y),
    \end{align*}
    and let $f$ be a vector-valued function $f: \Xi \times \wh \Xi_N \to \R^{N+2}$ 
    \[
        f(\xi, \xi') = \begin{pmatrix}
         \mathbbm{1}_{\hat \xi_i}(\xi') \\
         \vdots \\
         \mathbbm{1}_{\hat \xi_N}(\xi') \\
         f_\beta(\xi) \\
         g_\beta(\xi)
    \end{pmatrix}
    \]
    Then we have
    \[
        \begin{pmatrix}
            1/N \\
            \vdots \\
            1/N \\
            0 \\
            0
        \end{pmatrix} \in
        \mathrm{int}\left\{ 
            \EE_{\pi}[f(\xi, \xi')]: \pi \in \mc M_+(\Xi \times \wh \Xi_N)
        \right\}.
    \]
\end{lemma}
\begin{proof}[Proof of Lemma~\ref{lemma:slater-odd}]
    It suffices to show that for any 
    \[
        q \in \left(\frac{1}{2N}, \frac{3}{2N} \right)^N \times \left(-\frac{1}{4}, \frac{1}{4} \right)^2,
    \]
    there exists a nonnegative measure $\pi \in \mc M_+(\Xi \times \wh \Xi_N)$ such that $q = \EE_{\pi}[f(\xi, \xi')]$. The proof follows a similar argument as that of Lemma~\ref{lemma:slater-opp} by noticing that
    \begin{align*}
        \EE_{\pi}[g_\beta(\xi)] = (\wh p_{10}^N)^{-1} h_\beta(x_{10}) \sum_{i \in \mc I_{10}} q_i - (\wh p_{00}^N)^{-1} h_\beta(x_{00}) \sum_{i \in \mc I_{00}} q_i,
    \end{align*}
    and the specification of $x_{10}$ and $x_{00}$ can be achieved using similar steps.
\end{proof}

\begin{proof}[Proof of Proposition~\ref{prop:R-refor-odd}] 
To ease the exposition, we let the function $\Lambda: \mc A \times \mc Y \to \R^2$ be defined as
\begin{align*}
    \Lambda(a, y) =& \begin{pmatrix} (\wh p_{11}^N)^{-1} \mathbbm{1}_{(1, 1)}(a, y) - (\wh p_{01}^N)^{-1} \mathbbm{1}_{(0, 1)}(a, y) \\ (\wh p_{10}^N)^{-1} \mathbbm{1}_{(1, 0)}(a, y) - (\wh p_{00}^N)^{-1} \mathbbm{1}_{(0, 0)}(a, y)\end{pmatrix}  .
\end{align*}
Moreover, we define $f_\beta$ as in~\eqref{eq:f-def}, and additionally define $g_\beta$ as
\[
    g_\beta(x, a, y) = (\hat p_{10}^N)^{-1} h_\beta(x) \mathbbm{1}_{(1,0)}(a, y) - (\hat p_{00}^N)^{-1} h_\beta(x)  \mathbbm{1}_{(0,0)}(a, y).
\]
From the definition of $\mc{R}^\odd(\Pnom^N, \wh p^N)$, we have
\begin{align*}
        \mc{R}^\odd(\Pnom^N, \wh p^N) 
        =& \left\{ \begin{array}{cl}
        \Inf{\QQ \in \mc P}&\Wass(\Pnom^N, \QQ)^2 \\
        \st & (\hat p_{11}^N)^{-1} \EE_{\QQ}[h_\beta(X) \mathbbm{1}_{(1,1)}(A, Y)]  = (\hat p_{01}^N)^{-1} \EE_{\QQ}[h_\beta(X)  \mathbbm{1}_{(0,1)}(A, Y)] \\
        & (\hat p_{10}^N)^{-1} \EE_{\QQ}[h_\beta(X) \mathbbm{1}_{(1,0)}(A, Y)] = (\hat p_{00}^N)^{-1} \EE_{\QQ}[h_\beta(X)  \mathbbm{1}_{(0,0)}(A, Y)] \\
        & \QQ(A = a, Y = y) = \hat p_{ay}^N
        \quad\forall a \in \mc A,~ y\in \mc Y
        \end{array}
        \right. \\
        =& \left\{ \begin{array}{cll}
        \Inf{\pi}& \EE_{\pi}[ c\big( (X', A', Y'), (X, A, Y) \big)^2] \\
        \st & \pi \in \mc P((\mc X \times \mc A \times \mc Y) \times (\mc X \times \mc A \times \mc Y)) \\
        &\EE_{\pi}[f_\beta(X, A, Y)] = 0 \\
        &\EE_{\pi}[g_\beta(X, A, Y)] = 0 \\
        & \pi(A = a, Y = y) = \hat p_{ay}^N 
        &\forall a \in \mc A,~ y\in \mc Y \\
        &\EE_{\pi}[ \mathbbm{1}_{(\wh x_i, \wh a_i, \wh y_i)}(X', A', Y')] = 1/N &\forall i \in [N].
        \end{array}
        \right.
    \end{align*}
    To shorten the notations, we use $\Xi = \mc X \times \mc A \times \mc Y$ and $\wh \Xi_N = \{(\wh x_i, \wh a_i, \wh y_i)\}$. Moreover, define the vector $\bar q$ and the vector-valued Borel measurable function on $\Xi \times \wh \Xi_N$ as \[\bar q= \begin{pmatrix}
         0 \\ 
         0 \\
         1/N \\
         \vdots \\
         1/N
    \end{pmatrix}~ \hspace{.5cm}
    f(\xi, \xi') = \begin{pmatrix}
         f_\beta(\xi) \\
         g_\beta(\xi) \\
         \mathbbm{1}_{\hat \xi_i}(\xi') \\
         \vdots \\
         \mathbbm{1}_{\hat \xi_N}(\xi')
    \end{pmatrix}.\]
    By using the introduced notation, we can reformulate the above optimization problem as
    \[\inf \left\{
         \EE_{\pi}[ c(\xi, \xi')^2]:\pi \in \mc M_+(\Xi \times \wh \Xi_N),
        \EE_{\pi}[f(\xi, \xi')] = \bar q \right\}
    \] which is a problem of moments. 
    By Lemma~\ref{lemma:slater-odd}, the above optimization problem satisfies the Slater condition, thus the strong duality result~\cite[Section~2.2]{ref:smith1995generalized} implies that 
    \begin{align*}
        \mc{R}^\odd(\Pnom^N, \wh p^N) 
        =& \left\{ \begin{array}{cl}
            \sup & \ds  \frac{1}{N} \sum_{i=1}^N b_i \\
            \st & b \in \R^{N},~\dualvar \in \R,~\zeta \in \R \\
            &\ds\sum_{i=1}^N b_i \mathbbm{1}_{(\wh x_i, \wh a_i, \wh y_i)}(x', a', y') - \dualvar f_\beta(x, a, y)- \zeta g_\beta(x, a, y) \le c\big( (x', a', y'), (x, a, y) \big)^2 \\
            &\hspace{1cm}\quad \forall (x, a, y), (x', a', y') \in \mc X \times \mc A \times \mc Y
        \end{array}
        \right. \\
        =& \left\{ \begin{array}{cl}
            \sup & \ds \frac{1}{N} \sum_{i=1}^N  b_i \\
            \st & b \in \R^{N},~\dualvar \in \R,~\zeta \in \R \\
            &\ds b_i - \dualvar f_\beta(x_i, a_i, y_i)- \zeta g_\beta(x_i, a_i, y_i)  \le c\big( (\wh x_i, \wh a_i, \wh y_i), (x_i, a_i, y_i) \big)^2 \\
            &\hspace{1cm}\forall (x_i, a_i, y_i) \in \mc X \times \mc A \times \mc Y, \forall i \in [N]
        \end{array}
        \right. \\
        =& \Sup{\dualvar, \zeta} ~ \frac{1}{N} \sum_{i=1}^N \Inf{x_i \in \mc X} \left\{ \| x_i - \wh x_i \|^2 + \gamma f_\beta(x_i, \hat a_i,\hat y_i) + \zeta g_\beta(x_i,\hat a_i, \hat y_i) \right\},
    \end{align*}
    By definition of $f_\beta$, $g_\beta$ and the parameters $\lambda_i$, we have
    \begin{align*}
         \gamma f_\beta(x_i, \hat a_i,\hat y_i) + \zeta g_\beta(x_i,\hat a_i, \hat y_i) 
         = (\dualvar \lambda_i \mathbbm{1}_1 (\wh y_i) + \zeta
     \lambda_i \mathbbm{1}_0 (\wh y_i)) h_\beta(x_i).
    \end{align*}
    The proof is complete.
\end{proof}

\begin{proof}[Proof of Lemma~\ref{lemma:R-compute-odd}]
Because $[N] = \mc I_0 \cup \mc I_1$, we can write
\begin{align*}
    \mc{R}^\odd(\Pnom^N, \wh p^N) 
     =\Sup{\dualvar \in \R} ~\frac{1}{N} \sum_{i \in \mc I_1} \Inf{x_i \in \mc X} \left\{ \| x_i - \wh x_i \|^2 +  
     \dualvar \lambda_i h_\beta(x_i) \right\}    + \Sup{\zeta \in \R}~\frac{1}{N} \sum_{i \in \mc I_0} \Inf{x_i \in \mc X} \left\{ \| x_i - \wh x_i \|^2 +  
     \zeta \lambda_i h_\beta(x_i) \right\} .
\end{align*}
Note that the first supremum coincides with $\mc{R}^\opp(\Pnom^N, \wh p^N)$, and the second supremum is $U_N$. Under the Euclidean norm assumption, we can use Lemma~\ref{lemma:individual} to reformulate the inner infimum problems for $U_N$, which leads to~\eqref{eq:U_N}.
\end{proof}

\begin{proof}[Proof of Theorem~\ref{thm:limiting-odd}]
By applying a similar duality argument as in the proof of Theorem~\ref{thm:limiting-opp}, we can reformulate $\mc{R}^\odd(\Pnom^N, \wh p^N)$ as
\begin{align*}
    \mc{R}^\odd(\Pnom^N, \wh p^N) 
    &= \Sup{\gamma, \zeta} \EE_{\Pnom^N}\left[\Inf{\Delta} \left\{ \begin{array}{l}
    \gamma h_\beta(X + \Delta) \big( \frac{ \mathbbm{1}_{(1,1)}(A, Y) }{\wh p_{11}^N} - \frac{ \mathbbm{1}_{(0,1)}(A, Y)}{\wh p_{01}^N} \big) \\
    + \zeta h_\beta(X + \Delta) \big( \frac{ \mathbbm{1}_{(1,0)}(A, Y) }{\wh p_{10}^N} - \frac{ \mathbbm{1}_{(0,0)}(A, Y)}{\wh p_{00}^N} \big) + \|\Delta\|^2
    \end{array}
    \right\}
    \right] \\
    &= \Sup{\gamma, \zeta} \left\{ \frac{1}{\sqrt{N}}(\zeta H_0^N+ \gamma H_1^N) + \EE_{\Pnom^N}\left[\Inf{\Delta} \left( \begin{array}{l}
    \gamma [h_\beta(X + \Delta) - h_\beta(X)] \big( \frac{ \mathbbm{1}_{(1,1)}(A, Y) }{\wh p_{11}^N} - \frac{ \mathbbm{1}_{(0,1)}(A, Y)}{\wh p_{01}^N} \big) \\
    + \zeta [h_\beta(X + \Delta) - h_\beta(X)] \big( \frac{ \mathbbm{1}_{(1,0)}(A, Y) }{\wh p_{10}^N} - \frac{ \mathbbm{1}_{(0,0)}(A, Y)}{\wh p_{00}^N} \big) \\
    + \|\Delta\|^2
    \end{array}
    \right)
    \right]
    \right\}
\end{align*}
with the random variables $H_0^N$ and $H_1^N$ being defined as
\begin{align*}
        H_0^N &\Let \frac{1}{\sqrt{N}} \sum_{i=1}^N h_\beta(\wh x_i) \big( \frac{ \mathbbm{1}_{(1,0)}(\wh a_i, \wh y_i) }{\wh p_{10}^N} - \frac{ \mathbbm{1}_{(0,0)}(\wh a_i, \wh y_i)}{\wh p_{00}^N} \big),\\
        H_1^N &\Let \frac{1}{\sqrt{N}} \sum_{i=1}^N h_\beta(\wh x_i) \big( \frac{ \mathbbm{1}_{(1,1)}(\wh a_i, \wh y_i) }{\wh p_{11}^N} - \frac{ \mathbbm{1}_{(0,1)}(\wh a_i, \wh y_i)}{\wh p_{01}^N} \big).
\end{align*}
Notice that the condition
\[
    \PP\left( \Big \Vert\begin{pmatrix}
            \gamma_1 \\
            \gamma_0
        \end{pmatrix}^\top
        \Lambda(A, Y) \nabla h_\beta(X) \Big \Vert_* > 0\right) > 0
\]
is satisfied for any $(\gamma_0, \gamma_1) \not= 0$. Using the same argument as in the proof of \cite[Theorem~3]{ref:blanchet2019rwpi}, we can show that the optimal solution for $\gamma$, $\zeta$ and $\Delta_i$ belong to a compact set with high probability.
As $\wh p_{ay} - p_{ay} = o_{\PP}(1)$ for any $(a, y) \in \mc A \times \mc Y$, we have
\begin{align*}
    N \times \mc{R}^\odd(\Pnom^N, \wh p^N)
    &= 
    \Sup{\gamma, \zeta} \Bigg\{ \gamma H_1^N + \zeta H_0^N + \frac{1}{N} \sum_{i=1}^N \Inf{\Delta_i} \gamma \int_0^1 \nabla h_\beta\left(\wh x_i + t\frac{\Delta_i}{\sqrt{N}}\right) \cdot \Delta_i \mathrm{d}t ~ \times\\
    &\hspace{1cm}\begin{pmatrix}
            \gamma \\
            \zeta
        \end{pmatrix}^\top
        \begin{pmatrix} 
        p_{11}^{-1} \mathbbm{1}_{(1, 1)}(\wh a_i, \wh y_i) - p_{01}^{-1} \mathbbm{1}_{(0, 1)}(\wh a_i, \wh y_i) \\
        p_{10}^{-1} \mathbbm{1}_{(1, 0)}(\wh a_i, \wh y_i) - p_{00}^{-1} \mathbbm{1}_{(0, 0)}(\wh a_i, \wh y_i)
        \end{pmatrix}  + \| \Delta_i\|^2  + o_{\PP}(1)\Bigg\}.
\end{align*}
Using a similar argument, we can bound
\begin{align*}
        \| [\nabla h_\beta(x+\Delta) - \nabla h_\beta(x)] (p_{10}^{-1}\mathbbm{1}_{(1, 0)}(a, y) -  p_{00}^{-1} \mathbbm{1}_{(0,0)}(a, y))\|_2 \le 3\|\beta\|_*^2 M_0 \|\Delta\|
        \end{align*}
        for some constant $M_0$,
        and thus Assumption 6' in~\cite{ref:blanchet2019rwpi} is satisfied. If $H_0^N \xrightarrow{d.} H_0$ and $H_1^N \xrightarrow{d.} H_1$ for some random variables $H_0$ and $H_1$, then \cite[Lemma~4]{ref:blanchet2019rwpi} asserts that
\begin{align*}
        N \times \mc{R}^\odd(\Pnom^N, \wh p^N) \xrightarrow{d.}\Sup{\gamma, \zeta
        }\left\{\gamma H_1 + \zeta H_0 +  
        \EE_\PP\Big[\Big\Vert\begin{pmatrix}
            \gamma \\
            \zeta
        \end{pmatrix}^\top
        \begin{pmatrix} 
        p_{11}^{-1} \mathbbm{1}_{(1, 1)}(A, Y)- p_{01}^{-1} \mathbbm{1}_{(0, 1)}(A, Y) \\
        p_{10}^{-1} \mathbbm{1}_{(1, 0)}(A, Y)- p_{00}^{-1} \mathbbm{1}_{(0, 0)}(A, Y)
        \end{pmatrix} \nabla h_\beta(X) 
        \Big\Vert_*^2 \Big]\right\}.
    \end{align*}
    Using the same limiting argument as in the proof of Theorem~\ref{thm:limiting-opp}, we have the characterization of $H_1$ and $H_0$ as in the statement of the theorem.
\end{proof}

\section{Appendix - Auxiliary Result}
\label{sec:app-k-result}
\FloatBarrier
\begin{figure*}[hbt!]
    \centering
    \begin{subfigure}[h]{0.35\textwidth}
    \centering
    \includegraphics[width=\textwidth]{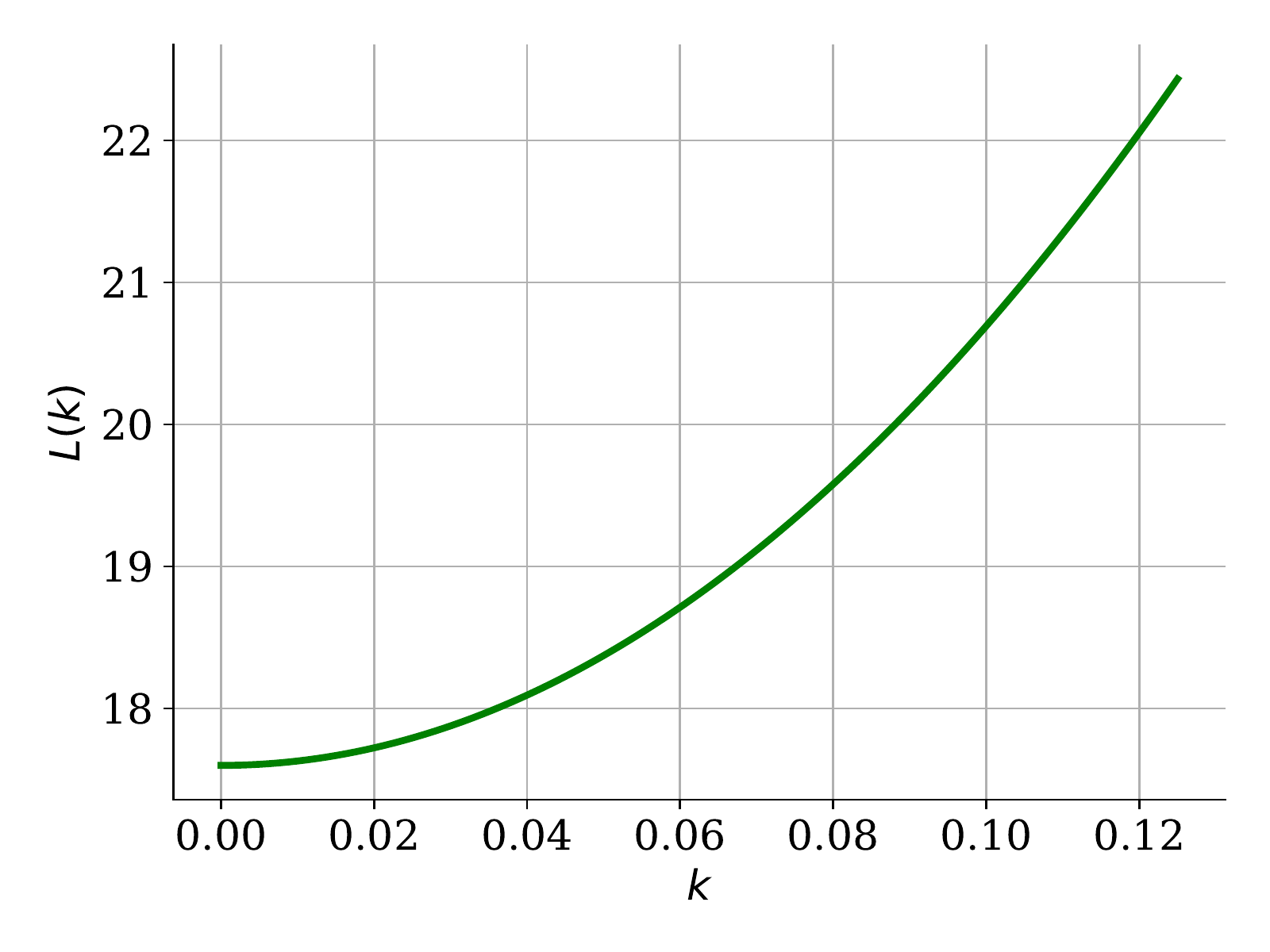}
    \caption{$\beta=(0, 1)^\top,~\wh x= (-2, 10)^\top,~\omega = 17.6$}
    \label{fig:Lk1}
    \end{subfigure}\hspace{2.5cm}
    \begin{subfigure}[h]{0.35\textwidth}
    \centering
    \includegraphics[width=\textwidth]{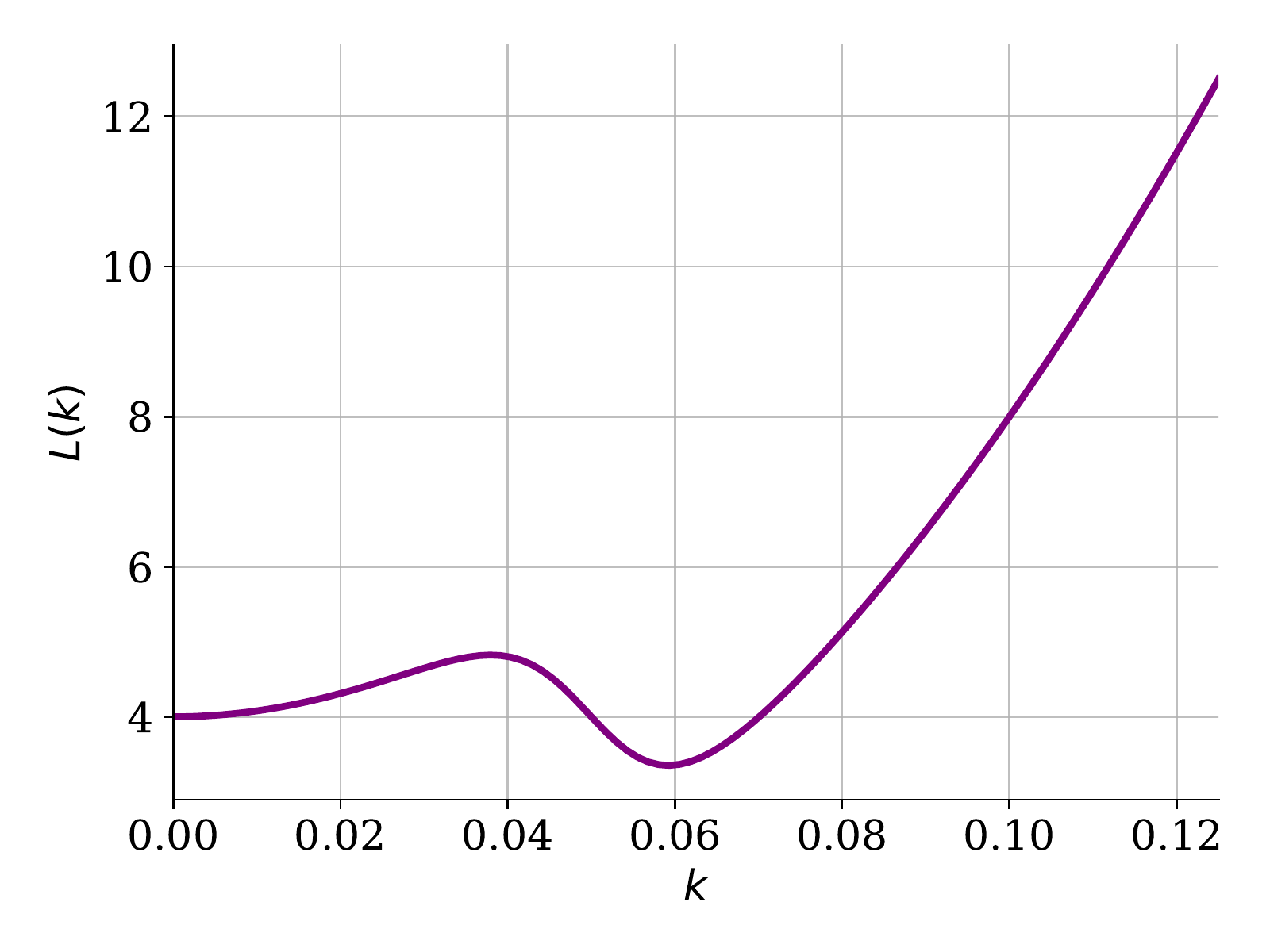}
    \caption{$\beta=(-5, 5)^\top,~\wh x = (3, 5)^\top,~\omega  = 4$}
    \label{fig:Lk2}
    \end{subfigure}
    \begin{subfigure}[h]{0.35\textwidth}
    \centering
    \includegraphics[width=\textwidth]{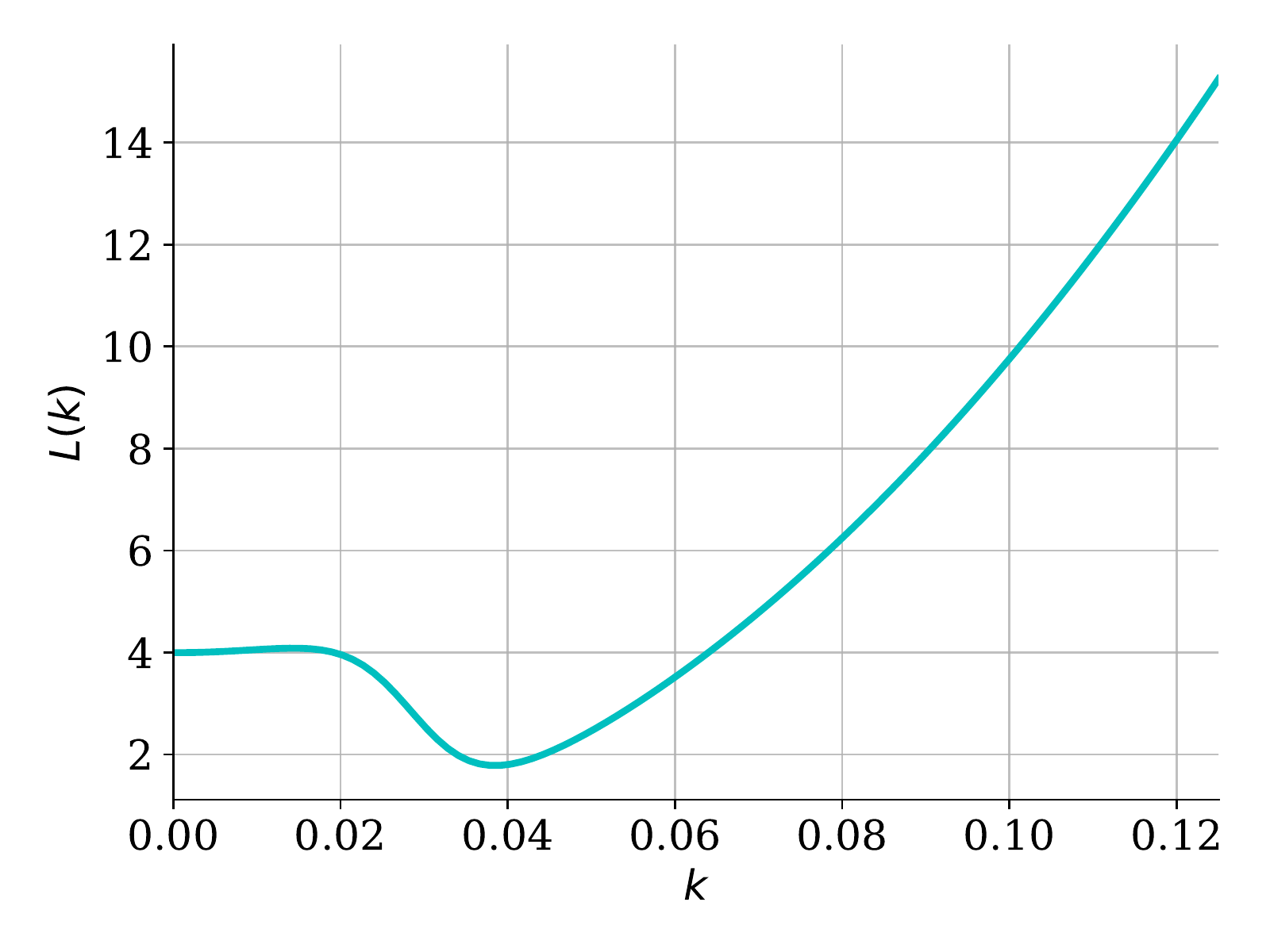}
    \caption{$\beta=(-6, 5)^\top,~ \wh x = (3, 5)^\top,~\omega  = 4$}
    \label{fig:Lk3}
    \end{subfigure}\hspace{2.5cm}
     \begin{subfigure}[h]{0.35\textwidth}
    \centering
    \includegraphics[width=\textwidth]{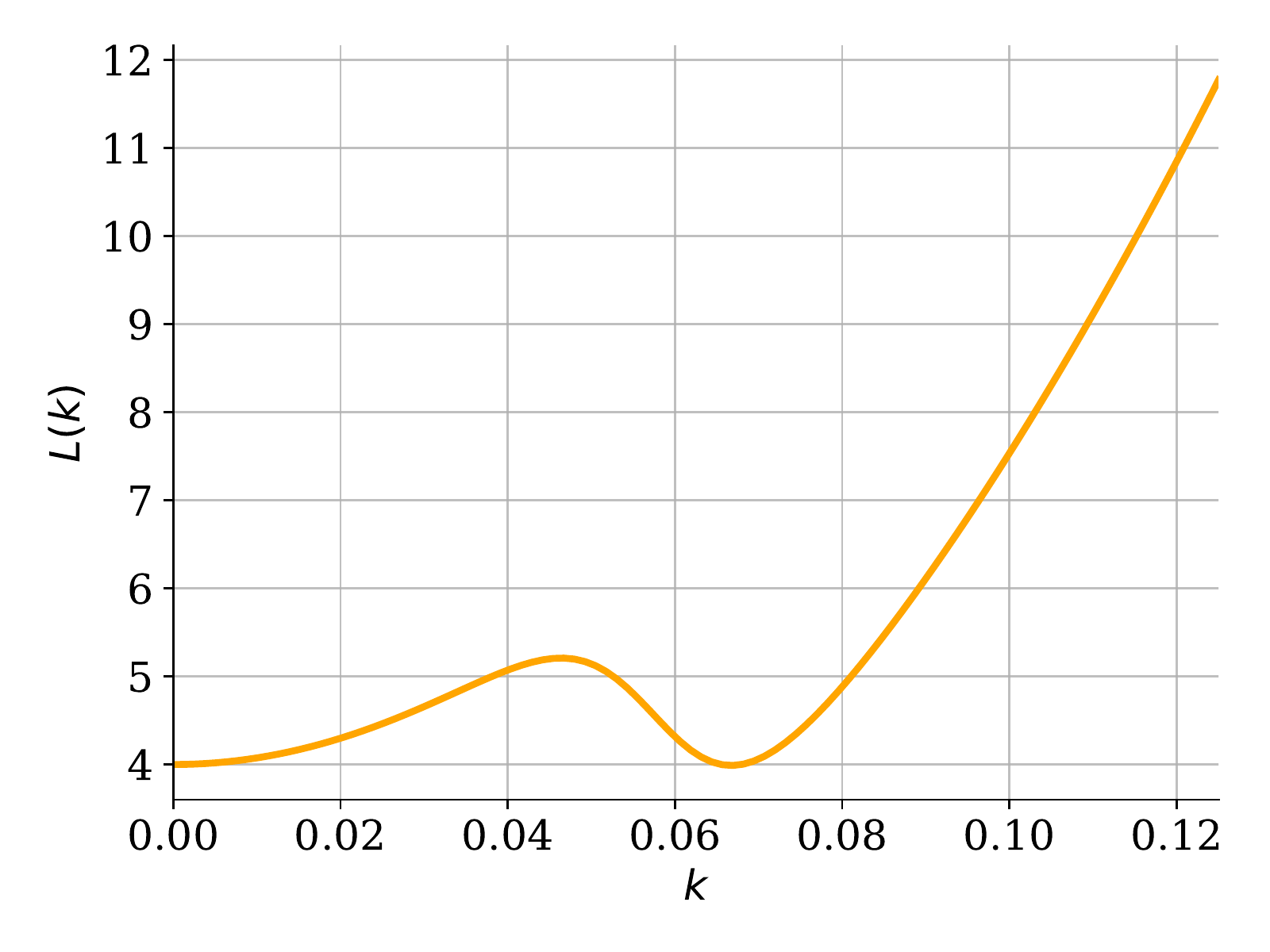}
    \caption{$\beta=(-4.7, 5)^\top,~\wh x = (3, 5)^\top,~\omega  = 4$}
    \label{fig:Lk4}
    \end{subfigure}
    \caption{Plots of $L(k)$ with respect to $k$ for different values of $\beta, \wh x$ and $\omega$.}
    \label{fig:non-convex-k}
\end{figure*}
The following lemma is used repeatedly to prove Lemmas~\ref{lemma:R-compute} and~\ref{lemma:R-compute-odd}.
\begin{lemma} \label{lemma:individual}
    For any $\omega \in \R$, $\wh x \in \R^p$ and $\beta \in \R^p$, we have
    \begin{align}
        \Inf{x \in \R^p}~\|x - \wh x\|_2^2 + \frac{\omega}{1 + \exp(-\beta^\top x)} 
        = \Min{ k \in [0, \frac{1}{8}]}~\omega^2 \| \beta\|_2^2 k^2 + \frac{\omega}{1 + \exp(-\beta^\top \wh x + k \omega \| \beta\|_2^2)}. \label{eq:min-L-of-k}
    \end{align}
\end{lemma}
\begin{proof}[Proof of Lemma~\ref{lemma:individual}]
    Any $x \in \R^p$ can be written using the orthogonal decomposition as $x = \wh x - k\omega \beta - k' \beta^{\perp}$ for some $k \in \R$, $k' \in \R$ and $\beta^{\perp}$ perpendicular to $\beta$, that is, $\beta^\top (\beta^\perp) = 0$. Optimizing over $x$ is equivalent to jointly optimizing over $k$, $k'$ and $\beta^\perp$ as
    \[
        \begin{array}{cl}
            \inf & \| k \omega \beta + k' \beta^\perp \|_2^2 + \ds \frac{\omega}{1 + \exp(-\beta^\top \wh x + k \omega \| \beta \|_2^2)} \\
            \st & k \in \R, \; k' \in \R, \; \beta^\perp \in \R^p,\; \beta^\top (\beta^\perp) = 0.
        \end{array}
    \]
    After extending the norm, and by noticing that the optimal solution in $k'$ and $\beta^\perp$ should satisfy $k' \beta^\perp = 0$, the above optimization problem is equivalent to
    \[
        \begin{array}{cl}
            \inf &  k^2 \omega^2 \|\beta\|_2^2 + \ds \frac{\omega}{1 + \exp(-\beta^\top \wh x + k \omega \| \beta \|_2^2)} \\
            \st & k \in \R.
        \end{array}
    \]
    Let $L(k)$ be the objective function of the above optimization problem, we have
    \begin{align*}
    \nabla_k L(k) &= 2\omega^2\|\beta\|_2^2 k - 
    \frac{\omega^2 \|\beta\|_2^2 \exp(-\beta^\top \hat x + k \omega \|\beta\|_2^2)}{(1+\exp(-\beta^\top \hat x + k\omega \|\beta\|_2^2))^2} \\
    &= \omega^2 \| \beta \|_2^2 \left( 2 k - \sigma(k) ( 1- \sigma(k)) \right),
    \end{align*}
    where for the purpose of this proof, we define $\sigma(k)$ as
    \[
        \sigma(k) \Let \frac{1}{1+\exp(-\beta^\top \hat x + k\omega \|\beta\|_2^2)} \in (0, 1).
    \]
    Notice that $\sigma(k) (1- \sigma(k)) \in (0, \frac{1}{4})$ for any value of $k \in \R$. Because $\nabla_k L(k)$ is continuous in $k$, $\nabla_k L(k) \le 0$ for any $k \le 0$, and $\nabla_k L(k) \ge 0$ for any $k \ge \frac{1}{8}$, one can conclude that there exists an optimal solution $k\opt$ that lies in the compact range $[0, \frac{1}{8}]$. This completes the proof.
\end{proof}
Let $L(k)$ be the objective function of the optimization problem~\eqref{eq:min-L-of-k}. Figure~\ref{fig:non-convex-k} visualizes several instances of $L(k)$ for different values of inputs $\beta, \wh x$ and $\omega$. Note that $L(k)$ is non-convex in $k$, and the optimizer of $L(k)$ is not necessarily unique as indicated in Figure~\ref{fig:Lk4}.

\section{Appendix - Numerical Results}
We use the synthetic experiment from~\cite{ref:zafar2015fairness} to generate unfairness landscapes provided in Figure~\ref{fig:unfairness_landscape}.
We set the true distributions of the class labels $\PP(Y=0) = \PP(Y=1)= 1/2$, and conditioning on $Y$, the feature $X$ is distributed as
    \begin{align*}
        X|Y=1 &\sim \mc N([2; 2], [5, 1; 1, 5]),\\
        X|Y=0 &\sim \mc N([-2; -2], [10, 1; 1, 3]).
    \end{align*}
Then, we draw sensitive attribute of each sample $x$ from a Bernoulli distribution, that is
    \[\PP(A = 1 | X = x') = pdf(x' |Y = 1)/(pdf(x'|Y = 1) + pdf(x'|Y = 0)),\] where $x'= [\cos(\pi/4),  \sin(\pi/4); \sin(\pi/4), \cos(\pi/4)]x$ is a rotated version of the feature vector $x$ and $pdf(\cdot | Y = y)$ is the Gaussian probability density function of $X$ given $Y=y$.
    
\bibliographystyle{siam} 
\bibliography{main.bbl}
\end{document}